\documentclass{article}

\usepackage{microtype}
\usepackage{graphicx}
\usepackage{booktabs} 

\usepackage{hyperref}

\usepackage[utf8]{inputenc} 
\usepackage[T1]{fontenc}    
\usepackage{url}            
\usepackage{booktabs}       
\usepackage{amsfonts}       
\usepackage{nicefrac}       
\usepackage{microtype}      

\usepackage{dsfont}

\usepackage{amsmath}
\usepackage{amssymb}
\usepackage{amsthm}
\usepackage{algorithmic}
\usepackage{algorithm}

\usepackage{bbold}

\usepackage{textcomp}
\usepackage{blindtext}
\usepackage{lipsum}
\usepackage{caption}
\usepackage{subcaption}
\usepackage[export]{adjustbox}

\usepackage{cuted}

\usepackage[title]{appendix}

\usepackage{mathtools}

\usepackage{xcolor}

\theoremstyle{plain}

\newtheorem{thm}{Theorem}
\newtheorem{prop}{Proposition}
\newtheorem{lem}{Lemma}
\newtheorem{cor}{Corollary}
\newtheorem{define}{Definition}

\newtheoremstyle{TheoremNum}
    {\topsep}{\topsep}              
    {\itshape}                      
    {}                              
    {\bfseries}                     
    {.}                             
    { }                             
    {\thmname{#1}\thmnote{ \bfseries #3}}
\theoremstyle{TheoremNum}
\newtheorem{thmn}{Theorem}
\newtheorem{propn}{Proposition}

\newcommand{\sr}{d^\pi}
\newcommand{\srstar}{d^{\pi*}}
\newcommand{\SR}{\psi}
\newcommand{\srhat}{\hat{\sr}}
\newcommand{\SRhat}{\hat{\SR}}
\newcommand{\V}{\mathcal{V}}
\newcommand{\Vol}{\operatorname{Vol}}
\newcommand{\Reg}{\operatorname{Reg}}

\newcommand{\SRprior}{\SR_0}

\newcommand{\B}{\mathcal{B}}
\newcommand{\Bd}[1]{\B_\delta(\tau_{#1})}

\newcommand{\traj}[1]{#1.\operatorname{traj}}
\newcommand{\state}[1]{#1.\operatorname{state}}

\newcommand{\argmax}{\operatorname{argmax}}
\newcommand{\actions}{\operatorname{actions}}
\newcommand{\action}{\operatorname{action}}
\newcommand{\parent}{\operatorname{parent}}
\newcommand{\child}{\operatorname{child}}
\newcommand{\stay}{\operatorname{stay}}
\newcommand{\SubtreeVol}{\operatorname{SubtreeVol}}
\newcommand{\subtree}{\operatorname{subtree}}

\newcommand{\KDtreeval}{\operatorname{KDTree\_value}}

\DeclarePairedDelimiter\floor{\lfloor}{\rfloor}

\usepackage[preprint]{icml2021}

\icmltitlerunning{Long-Horizon Exploration in MCTS}
\begin{document}

\twocolumn[
\icmltitle{Provably Efficient Long-Horizon Exploration in Monte Carlo Tree Search through State Occupancy Regularization}

\begin{icmlauthorlist}
\icmlauthor{Liam Schramm}{ru}
\icmlauthor{Abdeslam Boularias}{ru}
\end{icmlauthorlist}

\icmlaffiliation{ru}{Department of Computer Science, Rutgers University, New Brunswick, USA. This work is supported by NSF awards 1846043 and 2132972}
\icmlcorrespondingauthor{Liam Schramm}{lbs105@cs.rutgers.edu}

\icmlkeywords{Reinforcement Learning -- Theory, Monte Carlo tree search} 

\vskip 0.3in
]
\printAffiliationsAndNotice{}  



\begin{abstract}
Monte Carlo tree search (MCTS) has been successful in a variety of domains, but faces challenges with long-horizon exploration when compared to sampling-based motion planning algorithms like Rapidly-Exploring Random Trees. To address these  limitations of MCTS, we derive a tree search algorithm based on policy optimization with state occupancy measure regularization, which we call {\it Volume-MCTS}. We show that count-based exploration and sampling-based motion planning can be derived as approximate solutions to this state occupancy measure regularized objective. We test our method on several robot navigation problems, and find that Volume-MCTS outperforms AlphaZero and displays significantly better long-horizon exploration properties. 
\end{abstract}



\section{Introduction}

In robotics, sampling-based motion planning (SBMP) algorithms  are frequently used instead of reinforcement learning (RL) based methods such as Monte Carlo tree search (MCTS) for long-horizon exploration, due to challenges RL methods face in determining what regions may yield high rewards and how to reach them. 
While SBMP methods are highly efficient at exploration, they may be slow to converge to near-optimal paths, and do not provide a canonical way to either train or use neural networks to guide search ~\cite{McMahon2022}. Additionally, SBMP methods require much stronger assumptions than MCTS. They solve the problem of finding the shortest path to a goal region while avoiding obstacles, in a setting with continuous time setting with known and deterministic dynamics, while RL has been used in domains as wide-ranging as video games, autonomous driving, theorem proving, penetration testing, and power grid management ~\citep{Schrittwieser2020, osiński2021carla, lample2022hypertree, schwartz2019autonomous, zhang2019deep}.
We build on the recent work by ~\citet{GrillICML2020} to reveal a mathematical connection between MCTS, regularized policy optimization, and SBMP.  We then propose a family of MCTS algorithms based on policy optimization with state occupancy measure regularization, with strong exploration guarantees. 

The main contributions of this work are the following:
{\bf (1)}  We show that both the Voronoi bias of SBMP algorithms and the count-based exploration (CBE) method used in reinforcement learning can be derived as solutions to a state occupancy measure regularization objective.     
{\bf (2)} We prove that in search trees, for any convex loss function of state occupancy measure, $\mathcal{L}(\sr)$, $\mathcal{L}$ can be optimized by independently optimizing the policy at each node. This novel finding makes it possible and efficient to use MCTS-style algorithms for arbitrary regularization of the state occupancy measure. Notably, this is true \textit{only} for trees -- general Markov Decision Processes do not have this property \cite{Hazan2019}. 
{\bf (3)}  We derive {\it Volume-MCTS}, a variant of AlphaZero that uses state occupancy  regularization to encourage long-horizon exploration without making the stronger assumptions used in SBMP. 
We find that this method outperforms a range of reinforcement learning and planning algorithms, including AlphaZero and AlphaZero with CBE, on long-horizon exploration problems.    
{\bf (4)} We prove non-asymptotic high-probability bounds on Volume-MCTS's exploration efficiency. To the best of our knowledge, this is the first bound of this type to be proven for MCTS-family algorithms.

\section{Definitions} 
Let $M$ be a Markov Decision Process (MDP) with continuous state space $S$, continuous action space $A$, reward function $R\colon S \to \mathbb{R}$, discount factor $\gamma$, and deterministic transition function $\mathcal{T} \colon S \times A \to S$. 
We assume $S$ is bounded, measurable, and metrizable.
Let $T$ be the set of nodes in a search tree. Let $N$ be defined as $\mid \mid T \mid \mid$.
For any node $n \in T$, let subtree($n$) be the subtree of $n$.  
Let $\lambda$ be a regularization coefficient that scales as $\mathcal{O}\left (\frac{1}{\sqrt{N}} \right)$. 

\textbf{Node expansion:}
Let $\mathcal{M}(n)$ be the set of \textit{tree moves}, which are defined as the set of child actions, plus a ``stay'' action. 
Let the tree policy $\pi(n, a)$ be the probability of taking move $a$ when at node $n$.
Note that this is a policy over \textit{moves} we can take in the tree, not just actions. 
This shift from actions to tree moves is necessary as we will solve for the optimal probability with which to traverse the search tree, including stopping to expand a node. 
This means that we must explicitly include the choice to expand the current node as part of our policy search.
$\pi$ assigns a probability to taking each child action of the current node $n$, as well as the probability of expanding $n$. 
Let $V(n)$ be a value estimate for node $n$. 
We leave the precise estimation method for $V(n)$ open-ended.
Let $Q^\pi(\stay \mid n)$ be defined as $V(n)$,  and let $Q(a \mid s)$ be defined as $Q(a \mid n) = R(\state{\textit{n}}, a) + \gamma E_{a' \sim \pi(\cdot \mid child(n,a))}[Q(a' \mid child(n,a))]$ for all actions $a \in A(n)$.

\textbf{State Occupancy Measure:}
The state occupancy measure is the expected amount of time a policy $\pi$ will spend in a given state 
(or equivalently, the probability distribution of a policy's future states).
We repurpose the term slightly here, to focus on the density of the \textit{tree policy} future states in space. To do this, we look at the probability distribution of node expansions the tree policy induces on the tree. 
Let $P(n \mid \pi)$ be the probability that $n$ is reached when traversing the tree. 
Let $\sr(n')$ be the probability of expanding any given node $n'$ in the tree. 
Let $\sr(n' \mid n) = \frac{\sr(n')}{P(n\mid \pi)}$ be the probability of expanding any given node $n'$ in the tree, assuming we start at $n$ and traverse the tree according to $\pi$. 
Similarly, let $\sr(n' \mid n, a)  = \frac{\sr(n')}{P(n\mid \pi) \pi(a \mid n)} = \frac{\sr(n')}{P(\child(n, a) \mid \pi)}$ be the probability of expanding any node $n'$ if we additionally condition on taking action $a$ in state $s$. 
Let $\SR(T, \sr)$ be the state space $S$ density, the estimated density in the state space of samples drawn from the distribution $\sr$ over nodes $T$. We will see later that different density estimation methods will lead to different behavior, but we will primarily focus on the 1-nearest neighbor density estimator. 

\textbf{Empirical distributions:}
Let $\hat{\pi}(n, a)$ be the empirical policy, the fraction of times that each action $a$ has been selected from node $n$. $\srhat$ is the empirical state occupancy, defined as $\frac{1}{N}$ for all nodes. 
Let $\SRhat$ be the empirical state space density, the estimated density in the state space of nodes in the search tree.
Additionally, note that because $T$ is a tree, it induces a partial ordering over its nodes. We say that $n_i > n_j$ if $n_i$ is an ancestor of $n_j$. Similarly, we can say that an action $a_k > n_j$ if $a_k$ is higher up in the tree than $n_j$. We include this note because we will frequently need to sum over all the nodes in the subtree of a particular node or action, which we will write as $\sum_{n < n_i}$. 


\section{Background}

Our work seeks to bridge a range of approaches to long-horizon exploration. For this reason we begin with a review of four approaches to exploration that we show are connected: state space regularization, count-based exploration, SBMP, and Monte Carlo tree search. We build on the work by \citet{GrillICML2020} illuminating the link between MCTS-based algorithms and regularized policy optimization to formalize these connections.

\subsection{State Space Regularization}
\citet{Hazan2019} propose exploring by maximizing the entropy of a policy's state occupancy measure in the absence of a dense reward signal, and propose an algorithm that is guaranteed to be efficient in the tabular case. \citet{Seo2021} extend this idea by estimating state space entropy with random encoders and using this as an intrinsic reward in model-free RL. \citet{Yuan2022} further extend this method to the general class of Renyi divergences. Although the motivation of state space entropy maximization is very similar to our motivation of state space $f$-divergence regularization, the algorithms described are intended for a model-free setting, and do not apply to the MCTS setting. 

\subsection{Count-Based Exploration (CBE)}
One of the most successful methods in long-horizon exploration has been CBE ~\cite{Tang2017}. This family of methods gives an intrinsic reward to states, so that infrequently-visited states receive higher rewards. In this way, it is similar to performing UCB in the state space. ~\citet{Badia2020NGU} develop Never Give Up, an adaptive CBE method for Atari games that projects points into a latent space before doing the kernel density estimate. 
 Similar to our method, ~\citet{Machado2019} propose using the successor representation for CBE and find that this improves exploration in Atari environments. Agent 57 extends CBE to do long-term exploration for MuZero on Atari games, beating the human benchmark on all games ~\cite{Badia2020Agent57}. Although CBE is also used in this method, it is not applied to the tree itself -- the reward bonus depends only on the previous trajectory, and there is no information-sharing between tree nodes. To the best of our knowledge, no examples exist of MCTS-family algorithms that use CBE to share information about explored regions between nodes in different branches, and no previous works have developed a connection between CBE and $f$-divergence regularization of the state occupancy measure.

\subsection{Sampling-based Motion Planning}
SBMP algorithms like Rapidly-exploring Random Trees (RRT) are frequently employed in robotics for their efficient exploration. 
These methods sample random points in the state space and then expand the nearest point in the search tree in order to bias search towards unexplored regions ~\cite{LaValle1998, LaValle2006}. 
Since a node is expanded if and only if a point in its Voronoi region is sampled, SBMP algorithms are called \textit{Voronoi biased}, because the probability of expanding a node is proportional to the volume of its Voronoi region. 
While SBMP algorithms originally focused on feasible paths, recent work has focused on near-optimal planning. RRT* is an asymptotically-optimal variant of RRT for problems where a steering function is available~\cite{Karaman2011}. SST and AO-RRT are variants of RRT that are asymptotically optimal, even in the absence of steering functions or heuristics ~\cite{Li_Littlefield_Bekris_2016, Hauser2016}. PSST uses an RL-trained policy to guide the search while retaining SST's convergence guarantees ~\cite{psst}. 

\subsection{Monte Carlo Tree Search}
MCTS is a tree search strategy based on bandit algorithms ~\cite{MCTS2007}.
The most notable extensions to MCTS are the AlphaZero family, which includes AlphaGo, AlphaGo Zero, AlphaZero, MuZero, and Agent57 ~\cite{Silver2016, Silver2017, Silver2018, Schrittwieser2020, Badia2020Agent57}. 
These algorithms proceed in three main steps: selection, expansion, and backpropagation. 
For each iteration, the algorithm first selects the child action of the current node that maximizes the upper confidence bound. 
The algorithm selects actions to traverse the tree until it reaches a leaf node, which it then expands. 
Lastly, it updates the average value of each ancestor of the leaf node by backpropagating the new value estimate up the tree. 
While MCTS uses random rollouts to get value estimates, AlphaZero instead trains a neural network to estimate the value.  AlphaZero also trains a neural network policy $\pi_\theta$ to imitate the empirical policy $\hat{\pi}$. It then uses the policy-weighted upper confidence bound $UCB(s, a) = Q(s, a) + c \pi_\theta(a \mid s) \frac{\sqrt{N}}{N_a}$, where $Q$ is the value calculated by algorithm, $N$ is $s$'s visitation count, and $N_a$ is $a$'s visitation count. 
The policy focuses the tree towards branches that have been optimal in previous runs, leading to faster convergence. 

Standard MCTS only works for MDPs with finite action spaces. AlphaZero-Continuous is a minimal extension of AlphaZero that uses progressive widening and a continuous policy to extend AlphaZero to continuous environments~\cite{Moerland2018}.  
 Progressive widening samples new actions so that the number of actions at each node grows over time, typically as $O(\sqrt{N})$.
This is a standard approach for MCTS in continuous environments, but lacks the regret bounds of finite-action-space MCTS. 
Additionally, progressive widening does not use information from observed rewards to trade off exploration and exploitation, making it closer to $\epsilon$-greedy exploration than UCB. 
Furthermore, common progressive widening schedules lead to very rapid branching, causing the tree to have many short branches that explore the starting region much more than other regions. 
For this reason, we argue that it is better to explicitly consider node expansion in the objective, allowing the tree to grow as deeply or as broadly as needed. 
Since standard AlphaZero does not consider continuous state and action spaces, we will primarily focus on AlphaZero-Continuous as a representative of this family of methods. 

\subsection{MCTS as Regularized Policy Optimization:}
It has been shown in~\citet{GrillICML2020} that the $\frac{\log(N)}{\sqrt{N_a}}$ upper confidence bound in MCTS can be derived as a solution to a regularized policy objective. 
Consider the objective $\mathcal{L}(\pi) = \sum_a Q(s, a) \pi(a \mid s) - \lambda D_f(\pi \mid \mid \pi_0)$. 
The authors argue that it is possible to either solve this objective directly and sample from the resulting policy, or to approximate it by taking the action that maximizes the objective for the empirical policy. 
We will refer to these methods as the \textit{direct} and \textit{empirical} decision rules, respectively. 
The authors show that the empirical decision rule, $\argmax_a \frac{\partial}{\partial \hat{\pi}(a \mid s)} \sum_a Q(s, a) \hat{\pi}(a \mid s) - \lambda D_f(\hat{\pi} \mid \mid \pi_0)$, yields the $\frac{log(N)}{\sqrt{N_a}}$ upper confidence bound if $D_f$ is set to be the Hellinger divergence, where $f(t)=2(1-\sqrt{t})$. If $D_f$ is instead selected to be the reverse KL divergence, $f(t) = - \ln(t)$, then AlphaZero's upper bound of $Q(s, a) + c \pi_\theta(a \mid s) \frac{\sqrt{N}}{N_a}$ is recovered instead.

\section{RRT and Count-Based Exploration as Regularized Policy Optimization}

In this work, we are interested in what the direct and empirical decision rules are if the state space occupancy is regularized instead of the policy. 
We find that the direct decision rule yields a search algorithm that uses the Voronoi bias seen in SBMP algorithms, while the empirical decision rule results in a CBE reward. 
Since both of these methods are widely-used tools for learning and planning in long-horizon exploration problems, we hope that this generalized formalism will yield a family of algorithms that performs well at long-horizon exploration. 

\subsection{Connection to RRT}\label{Connection to RRT}
We use the direct decision rule described in the previous section to derive a search algorithm (Volume-MCTS) from a regularized return objective. 
Observe that in a search tree in which each node represents a state and each edge represents an action, each node $n$ is reached by a unique sequence of states $\traj{n} = (s_0, s_1, \dots s_{K})$, where $K$ is the node's depth in the search tree and $n$'s state is $s_K$. 
Consider a trajectory that begins with the state sequence $\traj{n}$, and then follows the policy $\pi_\theta$ after time $K$. 
Observe that the expected return of $n$'s trajectory is 
\begin{align*}
    \V(n) = E\left[ \left(\sum_{i=0}^{K-1} \gamma^i R(\traj{n}_i) + \gamma^K V^\pi_\theta(\traj{n}_K) \right) \right].
\end{align*}
We propose maximizing the expected return of nodes in the tree, minus a regularization term that rewards covering the state space as evenly as possible. 
Let $\sr(n)$ be the probability of expanding any node $n$ in the search tree, and let $\SR$ be the estimated density of $\sr$ in the state space. 
Then we seek to maximize the objective $\mathcal{L}(\sr) = E_{n \sim \sr}\left[\V(n) \right] - \lambda D_f(\SR(T, \sr) \mid \mid \SR_0)$ where $D_f$ is an $f$-divergence. 
 
 The intuition for this objective is to balance two goals. The first goal is to maximize reward, and the second is to evenly explore the state space.
 Since this formulation explicitly considers node expansion instead of a fixed progressive widening schedule, we choose a formulation of optimal return that allows us to compare nodes in different parts of the tree, as opposed to simply selecting actions from the same node. 
 Observe that in this formulation, nodes along an optimal path should score equally. 
 Once we solve for the optimal $\sr$, we can sample nodes from that distribution to expand. 

 Because $f$-divergences are convex, $\mathcal{L}(\sr)$ has a unique minimizer as long as $\SR$ is linear with respect to $\sr$. 
 To solve for this minimizer, we must first choose an $f$-divergence and a density estimation method for $\SR$. 
 For $f$, we follow \citet{GrillICML2020} in first examining the reverse KL divergence, $f(t) = - \ln{(t)}$. 
 For density, we introduce a generalization of the 1-nearest-neighbor estimator, which we call Partition Density Estimators. 
 We find that this class of estimators allows us to find a closed-form solution for $\sr$.

\begin{define} \label{def:PDE}
\textbf{Partition density estimator}    Let $D$ be 
a set of $m$ weighted points in $S$ with points $p_1 \hdots p_m$ and weights $w_1 \hdots w_m$. 

Then $\rho(D)$ is a partition density estimator iff, for each $s \in S$ except for a set of measure zero, 
\textbf{(1)} the gradient $\nabla_w \rho(D)(s)$ is non-zero for exactly one weight $w_i$, and 
\textbf{(2)} $\rho(D)(s) = w_i g(D, s)$ for some function $g$. 
\end{define}

We call this a partition density estimator because it allows us to partition the space into regions which only depend on one point in the dataset (except for a zero-measure boundary between these partitions). 
For instance, consider a weighted variant of the 1-nearest neighbor density estimator, where the estimated density at a point is proportional to the weight of the nearest neighbor. This is a partition density estimator because the density at any state $s$ only depends on the location and weight of $s$'s nearest neighbor.
Observe that any partition density estimator is linear with respect to the weights. Therefore, convex functions of such density estimators will also be convex with respect to the weights. 

\begin{define} \label{def:AV}
\textbf{Associated Volume}    
 Let $D$ be a set of $m$ weighted points in $S$ with points $p_1 \hdots p_m$ and weights $w_1 \hdots w_m$, and let $\rho(D)$ be a partition density estimator. 
Let $\mu$ be a probability measure on $S$. 
 Then the associated region $\Reg(i)$ is the set of all $s \in S$ for which $\nabla_{w_i} \rho(D)(s)$ is non-zero. The associated volume $\Vol(i)$ of any point $p_i$ is $\mu(\Reg(i))$, the measure of $i$'s associated region.
\end{define}

Suppose that $D$ has uniform weights, and $\mu$ is a uniform probability measure. Then observe that the 1-nearest neighbor density estimator is a partition density estimator, and the associated volume of any point $i$ in the data set is the volume of its Voronoi region.

\begin{prop} \label{prop:prop1}
    Suppose $f(t) = - \ln{t}$ and $\SR(T, \sr)(s)$ is a partition density estimator, where the associated measure of any node $n$ is $\Vol(n)$. 
    Then $\mathcal{L}$ has a unique optimizer $\srstar$, such that $\srstar(n) =  \frac{\lambda}{\alpha - \V(n)} \Vol(n)$, where $\alpha$ is a constant that makes $\srstar$ a proper probability distribution.
\end{prop}


\textbf{Proof}: 
The unique solution to a convex function of a probability distribution can be found by setting the gradient equal to a constant $\alpha$ and then solving for $\pi$, where $\alpha$ normalizes the solution. We find that the density estimator $\SR$ cancels out, meaning the solution is independent of the choice of $\SR$. 
Details are provided in Appendix~\ref{Uniqueness Proof}. 

Although $\alpha$ does not have a closed form solution, it is possible to find upper and lower bounds for it. Since $\sr(n)$ monotonically decreases as $\alpha$ increases, it is simple to calculate $\alpha$ numerically using Newton's method. 
When $\alpha$ is known, we can sample from $\sr$.

\textbf{Relation to RRT}: Consider the case in which $\V(n) = 0$ for each node $n$ (or equivalently, the limit in which the regularization coefficient $\lambda$ is large). Then $\alpha = \lambda$, so $\sr(n) = \Vol(n)$. 
If we choose 1-nearest-neighbor as our density estimator, 
then the probability of expanding any given node $n$ is the volume of $n$'s Voronoi region. This is the same probability of node expansion used in RRT ~\cite{LaValle1998}.
Thus we can see that the $\srstar(n) =  \frac{\lambda}{\alpha - \V(n)} \Vol(n)$ sampling distribution behaves like RRT when $\lambda$ is large, 
but behaves more greedily and has a lower probability of sampling suboptimal nodes  as $\lambda$ decreases over time.
Unlike RRT, we makes no assumptions about the reward structure, and can apply this sampling distribution to any continuous state- and action-space RL problem, whereas RRT is limited to path-planning problems. 
This connection to RRT and the well-motivated generalization of the Voronoi bias to RL is the first major contribution of our work.



\subsection{Connection to Count-based Exploration}

For this approach, we consider the empirical decision rule. We prove the following: 

\begin{prop} \label{prop:prop2}
Suppose $D_f$ is chosen to be the Hellinger distance, $f(t) = 2(1-\sqrt{t})$, and $\SRhat$ is chosen to be a kernel density estimator, $\SRhat((T, \srhat))(s) = \sum_{i \in T} \srhat(i) k(\state{i}, s)$. Additionally, suppose $\SRprior$ is the uniform distribution over the state space. 
Let $R_{CBE}(n) = \sqrt{\frac{1}{\sum_{i \in T} k(\state{i}, \state{n})}}$, the CBE reward described in ~\citet{Badia2020NGU}. 
Then,
    \begin{align*}
        a   &= \argmax_n \frac{\partial}{\partial \srhat(n)} E_{n' \sim \srhat}[\V(N')]  - \lambda D_f(\SRhat \mid \mid \SRprior ) \\
            &   \approx \argmax_a Q(s, a)  + c E_{n' \sim \subtree(a)}\left[ R_{CBE}(n') \right]. 
    \end{align*}
\end{prop}
\textbf{Proof}: 

This derivative simplifies to: 
\begin{align*}
    &\frac{\partial}{\partial \srhat(n)}  \mathcal{L}(\srhat) =  \V(n)  \\
    &+ \lambda \int_S k(\state{n}, s) \sqrt{\frac{\SRprior(s)}{\sum_{i \in T} \srhat(i)  k(\state{i}, s)}} ds
\end{align*}

We can approximate the integral by taking a linear approximation of $\sqrt{\frac{\SRprior(s)}{\sum_{i \in T} \srhat(i)  k(\state{i}, s)}}$ about the point $s=\state{n}$, where $k(\state{n}, s)$ is largest. This reduces to:


\begin{align*}
\frac{\partial}{\partial \srhat(n)}  \mathcal{L}(\srhat) &\approx \V(n)  + c \sqrt{\frac{1}{\sum_{i \in T} k(\state{i}, \state{n})}} \\
&= \V(n)  + c R_{CBE}(n')
\end{align*}

The empirical decision rule is then 
\begin{align*}
    &\argmax_a \frac{\partial}{\partial \pi( a | n)} \mathcal{L}(\srhat) \\
    &= \argmax_a Q(s, a)  + c E_{n' \sim \subtree(a)}\left[ R_{CBE}(n') \right]
\end{align*}






The full derivation is in Appendix \ref{CBE Proof}.

 Observe that $Q(s, a)$ is the empirical average of future rewards calculated by MCTS, and $E_{n' \sim \subtree(a)}\left[ R_{CBE}(n') \right] $ is the empirical average of future exploration rewards calculated by MCTS. In other words, $E_{n' \sim \subtree(a)}\left[ R_{CBE}(n') \right] $ is the value for CBE rewards. 

\section{Volume-MCTS Algorithm}

In section \ref{Connection to RRT}, we proved $\srstar(n) =  \frac{\lambda}{\alpha - \V(n)} \Vol(n)$ is the optimal solution to the objective $\mathcal{L}(\pi)$, but this does not show how to calculate $\srstar$ efficiently. 
This formulation also does not make the connection between this algorithm and traditional MCTS obvious. 
To address this, we first prove that it is possible to sample from $\srstar$ without explicitly solving for it by instead solving for the optimal tree policy $\pi$ at each node. This is non-trivial to show, as convex functions of the state occupancy measure are not necessarily convex with respect to the policy. For instance, ~\citet{Hazan2019} show that the entropy of the state occupancy measure is non-convex with respect to the policy, and in fact has local minima. This means that regularization of the state occupancy measure is difficult to solve for in general MDPs, and may not work in combination with standard methods such as policy gradient methods. However, we find that for trees specifically, it is possible to show that solving for the locally optimal policy at each node is sufficient to find $\srstar$: 

\begin{thm} \label{thm:theorem1}
    For any convex loss function $\mathcal{L}(\sr)$, $\mathcal{L}$ is convex with respect to $\pi(a \mid n)$ for all nodes $n$ and moves $a$. Furthermore, $\mathcal{L}$ is minimized if and only if for every node $n$, $\frac{\partial \mathcal{L}}{\partial \pi(a \mid n)}$ is constant for all moves $a$. 
\end{thm}
\textbf{Proof}: 

Recall that a convex function $\mathcal{L}$ of a probability distribution $P$ is minimized iff $\frac{\partial}{ \partial P(n)} = \alpha$ for all $n$, where $\alpha$ is a constant. Hence, when $\mathcal{L}(\sr)$ is minimized, $\frac{\partial}{ \partial \sr(n)} \mathcal{L}(\sr) = \alpha$ for all $n$. 
We find that:
\begin{align*}
    \frac{\partial}{\partial \pi(a \mid n)} \mathcal{L} & = P(n) E_{n' \sim \sr(\cdot \mid n, a)}\left[\frac{\partial f}{\partial \sr(n')}\right]
\end{align*}
This condition implies  that $\frac{\partial}{ \partial \sr(n)} \mathcal{L}(\sr) = \alpha$ for all $n$ iff $\frac{\partial}{\partial \pi(a \mid n)} \mathcal{L}$ is constant for all $a$ at every node $n$. Hence, the minimization problem for $\sr$ is solved iff the convex loss with respect to the policy at each node is minimized. 

The full proof is provided in Appendix \ref{Convexity Proof}.

This makes it possible to solve for arbitrary convex losses of $\sr$ by solving for the optimal $\pi$ at each node $n$. We use this theorem to derive a closed-form expression for the optimal policy at each node. 

\begin{thm} \label{thm:theorem2}
     Suppose $f(t) = - \ln{t}$ and $\SR(T, \sr)(s)$ is a partition density estimator. Then $\mathcal{L}(\pi)$ has a unique optimizer $\pi^*$, such that  $\pi^*(a \mid n) = \frac{\lambda \SubtreeVol(a)}{\alpha - \gamma^{d(n)} P(\text{$n$ reached } \mid \pi) Q^{\pi^*}(a \mid n)}$, where $\SubtreeVol(a)$ is the total associated volume of all nodes in the subtree of $a$,  $d(n)$ is the depth of $n$ in the search tree, 
    $P(\text{$n$ reached} \mid \pi^*)$ is the probability that we reach $n$ when traversing the tree according to $\pi^*$, 
    and $\alpha$ is whatever constant normalizes $\pi^*$.  
    Additionally, $\pi^*$ is the unique distribution that induces $\srstar$ as the state occupancy measure. 
\end{thm}

\textbf{Proof}: Details in Appendix \ref{Convexity Proof}.

We propose to expand the search tree by traversing until we sample an action to expand the current node. 
Unlike RRT, we cannot expand the tree by sampling the state space and expanding the nearest node, because our expansion probability depends on $\V(n)$ as well as the density $\SRprior$. 
Instead, we explicitly solve for the value of $\pi$ that optimizes $\mathcal{L}(\pi)$. To do this, we select a density estimator that makes $\SubtreeVol(a)$ simple to calculate.
RRT implementations typically use a data structure called a k-d tree to store their list of visited states and quickly find the approximate nearest neighbor. These trees effectively act as binary search trees for $k$-dimensional spaces. Each non-leaf k-d node defines a hyperplane that splits the space along one dimension. All states to one side of the hyperplane are stored in the left child node, and all states on the other side are stored in the right child node. Each child node splits the space again and divides the stored states between its children. This repeats until we reach a leaf node, which has only one state in its region. We can add nodes to this structure by traversing the tree until we find a leaf node, dividing its region in two, and giving that node two child leaf nodes to store its original state and the new state. The region covered by a node is called a k-d region. Since k-d regions are always rectangular, their volumes are easy to calculate. The k-d region of a leaf node always contains exactly one state, so the density estimator $\rho_{kd}(\pi)(x) = \frac{\SR(x)}{\text{kd region volume}(x)}$ is a partition density estimator. 

Our algorithm is detailed in Algorithm \ref{alg:Volume MCTS}. Starting at the root, we calculate the optimal tree policy at the current node, and then sample from the policy to walk down the tree until we select a node $n$ to expand. After we select $n$, we sample a new action $a$ to add to the tree, execute this action, and find the next state $s'$. We then add $s'$ to both the search tree and the k-d tree, find the volume of its k-d region, and get a value estimate. We then backpropagate the value estimate up the search tree. Additionally, at each step in the search tree backpropagation, we backpropagate the value up the k-d tree. We also make a slight approximation to the solution derived above. While it is possible to calculate $\pi^*(a \mid n) = \frac{\lambda SubtreeVol(a)}{\alpha - \gamma^{d(n)} P(\text{$n$ reached } \mid \pi) Q^{\pi^*}(a \mid n)}$ using only local information, the entire tree would need to be recalculated every time $\lambda$ changes, and $\lambda$ changes every iteration. Instead we use the k-d tree's approximation of $Q$ in place of $Q^{\pi^*}$, which does not require us to recalculate the tree. This approximation is preferable to approximating $Q$ using the MCTS method, where $\hat{Q}$ is the average of the node evaluations in $a$'s subtree, because it allows for information sharing between nodes in different subtrees. This means that $Q(s, a)$ can converge to $Q^{\pi^*}(s, a)$, even if $a$ is not sampled. 
For instance, if a good trajectory beginning at state $s$ is discovered, then all nodes near $s$ will have their values increase, without needing to be expanded first. 

\subsection{Guarantees}
Volume-MCTS's state-space exploration allows us to derive stronger exploration guarantees than are possible for MCTS. Under mild conditions, we provide non-asymptotic, high-probability bounds  on the number of expansions needed to reach a given region in state space. We begin by defining $\delta$-controllability, which we use as a weak notion of continuity. 

\begin{define}
    \textbf{$\delta$-controllable}: 
    Let $M$ be an MDP with action space $A$, bounded and measurable state space $S$, deterministic transition function $\mathcal{T}(s_i, a_i)$, and discount factor $\gamma$.
    Let $d_A$ be the dimensionality of $A$. 
    Let $\tau$ be a trajectory in $M$. 
    Let $s_i$ be the $i$-th state in $\tau$. 
    Let $\B_\delta(s_i)$ be a ball of radius $\delta$ about $s_i$. 

    Then $\tau$ is $\delta$-controllable iff there exists a constant $\sigma > 0$ such that for each state $s_i$ in $\tau$, there exists a region in action space $A_i$ with measure at least $\sigma \delta^{d_A}$ such that if a state $s_i' \in \B_\delta(s_i)$ and $a_i' \in A_i$, then $\mathcal{T}(s_i', a_i') \in \B_\delta(s_{i+1})$.    
\end{define}

Intuitively, if we have a point close to a state $s_i$ in the trajectory $\tau$, then we have a lower bounded chance of sampling an action that stays close to $\tau$ at the next state.
It is a strictly weaker assumption than the notion of $\delta$-robustness defined in \citet{Li_Littlefield_Bekris_2016}.

\begin{thm}
    Let $\tau$ be a $\delta$-controllable trajectory, with states $s_0 ... s_L$. 
    Let $d_A$ be the dimension of the action space. 
    Let $\Bd{i}$ be the $\delta$-ball around $s_i$, the $i$-th state in $\tau$. 

    Then the probability that $\Bd{i}$ will be reached after $N$ expansions is lower-bounded by $1-\frac{\Gamma(i, \frac{1}{2} |\B_{\frac{\delta}{5}}| \sigma \delta^{d_A} c(1-\gamma) (\sqrt{N_1} - c (1-\gamma))))}{\Gamma(i)}$, where $\Gamma$ is the incomplete Gamma function. 
\end{thm}

\textbf{Proof}: Details in Appendix \ref{Efficiency Proof}.

Since this bound takes the form of a Gamma distribution, it is easy to conclude the following corollary. 
\begin{cor}
    With probability > 0.5, $\Bd{i}$ will be reached after $c^2(1-\gamma)^2 (\frac{1}{2} i |\B_{\frac{\delta}{5}}| \sigma \delta^{d_A}  + 1)^2$ steps. 
\end{cor}
This means that any region on a $\delta$-controllable trajectory of length $t$ will be reached after $O(t^2)$ steps with probability $> 0.5$.  While several MCTS variants have known regret bounds, to the best of our knowledge, this is the first high-probability bound on long-horizon exploration speed for an MCTS-family algorithm, and is a contribution of this work. 

\subsection{Tree Search Algorithm}

\begin{algorithm}[tbh]
  \caption{Volume Monte Carlo Tree Search}
  \label{alg:Volume MCTS}
\begin{algorithmic}[1]
\STATE {\bfseries Have:} Regularization coefficient $\lambda$, KDTree; 
\STATE {\bfseries Input:} Node $n$ with child branches $a_1, \dots, a_k$
\STATE $Q(``\stay'') \leftarrow \KDtreeval(n.state)$;
\FOR {$a_i \in \{a_1, \dots, a_k\}$}
\STATE $Q(a_i) \leftarrow \KDtreeval(\mathcal{T}(n.state, a_i))$;
\ENDFOR
\FOR {$a_i \in \{a_1, \dots, a_k\} \cup \{``\stay''\} $}
\STATE $\pi(a_i \mid n) \leftarrow \lambda \frac{1}{\alpha - \gamma^d P(\text{$n$ is reached}) Q(a_i)} \Vol(a_i)$; 
\ENDFOR
\STATE Sample next move $a\sim\pi(.|n)$;
\IF{$a = ``\stay''$} \STATE value $\leftarrow$ Expand($n$);
\ELSE \STATE value $\leftarrow$ Search($a$);
\ENDIF
\STATE value $\leftarrow r + \gamma \times value$;
\STATE $n$.value\_sum $\leftarrow$ $n$.value\_sum + value;
\STATE $n$.visit\_count $\leftarrow$ $n$.visit\_count + 1;
\STATE KDBackprop(value, $n$.state);
\RETURN{} value; 
\end{algorithmic}
\end{algorithm}

Our estimator for $V(s)$ is derived from the k-d tree. We search the k-d tree for the node that contains only the state $s$, find the node halfway up the tree, and use the average value of all states in that k-d region as the value estimator. The intuition is that this makes a good bias-variance tradeoff. 
Nodes near the root of the k-d tree average the values of many states from a large region. This means they have low variance, but high bias (because the states impacting the estimate may be far away and have different values). 
Nodes near the leaves of the k-d tree have few points from a small region, so the average value of their states is high-variance and low-bias. As more states are explored and the k-d tree grows deeper, nodes halfway down the k-d tree will have regions with volumes that go to zero (implying low bias), and contain many states (implying low variance). Based on this, we conjecture, but do not prove, that this method is a consistent estimator of the true value. We can make this estimate in $O(\ln{n})$ time because it only requires one query, as long as we keep track of the average value of each k-d node in the tree. The KDBackprop function keeps these k-d tree values up to date after each node expansion. The full KDTree\_value and KDBackprop algorithms, as well as discussion of the challenges to proving consistency, are described in Appendix \ref{Algorithm Details}.

\subsection{Expansion, Value Estimation $\&$ K-D Tree Backprop}

We sample new actions from a policy $\pi_\theta$, which is represented by a neural network. The value estimates obtained during expansion also utilize a neural network $V_\theta$. 
After each episode, we train the value function and policy. The value function is trained to predict the value given by the k-d tree. Ideally, we would train the policy to minimize the objective described earlier. However, it is not trivial to find a closed form representation of the state occupancy divergence from a new policy. Instead, we use ordinary $f-$divergence regularization for the policy. This gives us the following loss function, 
\begin{eqnarray*}
    \mathcal{L} &= c_V (V_\theta - \hat{V}_{kd})^2  + c_{KL} \lambda KL(\pi_0 \mid \mid \pi_\theta) \\
    &- c_{A} \sum_{a \in \actions} A(a) \pi_\theta(a),
\end{eqnarray*}
in which $V_\theta$ is the value network, $\hat{V}_{kd}$ is the k-d tree value, $\pi_\theta$ is the policy network, $\pi_0$ is the baseline policy (for instance, a unit Gaussian), $KL$ is the KL divergence between the two policies, $\lambda$ is the regularization coefficient, and $A(a)$ is the advantage of action $a$. $c_V, c_{KL},$ and $c_{A}$ are hyperparamters. 

\subsection{Extension to Non-deterministic Environments and Action-dependent Rewards}
So far, our approach has made two significant assumptions: we assume that the rewards depend only on the state and not the action, and that the dynamics are deterministic. Here, we would like to briefly note that it is possible to remove these assumptions with some small changes. Action-dependent rewards can be accounted for by regularizing the policy. Stochastic dynamics can be handled using a technique like Double Progressive Widening \citep{bertsimas2014doubleprogressivewidening}. Details are provided in Appendix \ref{Extension}.


\section{Experiments}

\begin{table*}[t]
{\tiny
\begin{tabular}{ |p{3cm}||p{1.2cm}|p{1.2cm}|p{1.2cm}|p{1.2cm}|p{1.2cm}|p{1.2cm}|p{1.2cm}|p{1.2cm}|  }
 \hline
 \multicolumn{9}{|c|}{Geometric Maze with Continuous State and Action Spaces (No Training)} \\
 \hline
 Maze Size  & 2 & 3 & 4 & 5 & 6 & 7 & 8 & 9  \\
 \hline
 \hline
AlphaZero&36.0 $\pm$ 6.0&6.0 $\pm$ 4.0&7.0 $\pm$ 5.0&0.0 $\pm$ 0.0&0.0 $\pm$ 0.0&0.0 $\pm$ 0.0&0.0 $\pm$ 0.0&0.0 $\pm$ 0.0\\
AlphaZero w/ CBE&37.0 $\pm$ 1.0&38.0 $\pm$ 3.0&25.0 $\pm$ 4.0&16.0 $\pm$ 4.0&6.0 $\pm$ 3.0&2.0 $\pm$ 2.0&0.0 $\pm$ 0.0&0.0 $\pm$ 0.0\\
Volume-MCTS&\textbf{49.0 $\pm$ 0.0}&\textbf{46.0 $\pm$ 0.0}&\textbf{43.0 $\pm$ 1.0}&\textbf{38.0 $\pm$ 1.0}&\textbf{33.0 $\pm$ 1.0}&\textbf{31.0 $\pm$ 1.0}&\textbf{22.0 $\pm$ 4.0}&\textbf{7.0 $\pm$ 3.0}\\
OL AlphaZero&\textbf{49.0 $\pm$ 0.0}&0.0 $\pm$ 0.0&0.0 $\pm$ 0.0&0.0 $\pm$ 0.0&0.0 $\pm$ 0.0&0.0 $\pm$ 0.0&0.0 $\pm$ 0.0&0.0 $\pm$ 0.0\\
 \hline
 
 \hline
 \multicolumn{9}{|c|}{Geometric Maze with Continuous State and Action Spaces (After Training)} \\
 \hline
 Maze Size & 2 & 3 & 4 & 5 & 6 & 7 & 8 & 9  \\
 \hline
 \hline
AlphaZero &45.0 $\pm$ 1.0&42.0 $\pm$ 1.0&20.0 $\pm$ 6.0&0.0 $\pm$ 0.0&0.0 $\pm$ 0.0&0.0 $\pm$ 0.0&0.0 $\pm$ 0.0&0.0 $\pm$ 0.0\\
AlphaZero w/ CBE &45.0 $\pm$ 1.0&41.0 $\pm$ 1.0&32.0 $\pm$ 4.0&0.0 $\pm$ 0.0&1.0 $\pm$ 1.0&0.0 $\pm$ 0.0&0.0 $\pm$ 0.0&0.0 $\pm$ 0.0\\
Volume-MCTS &44.0 $\pm$ 5.0&\textbf{47.0 $\pm$ 0.0}&\textbf{45.0 $\pm$ 1.0}&\textbf{40.0 $\pm$ 1.0}&\textbf{38.0 $\pm$ 1.0}&\textbf{25.0 $\pm$ 4.0}&\textbf{22.0 $\pm$ 4.0}&\textbf{26.0 $\pm$ 4.0}\\
OL AlphaZero &\textbf{49.0 $\pm$ 0.0}&47.0 $\pm$ 0.0&0.0 $\pm$ 0.0&0.0 $\pm$ 0.0&0.0 $\pm$ 0.0&0.0 $\pm$ 0.0&0.0 $\pm$ 0.0&0.0 $\pm$ 0.0\\
 \hline
\end{tabular}
}
\vspace{-0.2cm}
\caption{ Average rewards and standard errors on geometric navigation environments. All methods use 5000 total rollouts per episode.}
\vspace{-8pt}
\label{geom_perf}
\end{table*}

\begin{table}
{\tiny
\begin{tabular}{ |p{1.25cm}||p{0.9cm}|p{0.9cm}|p{0.9cm}|p{0.9cm}|p{0.9cm}|}
 \hline
 \multicolumn{6}{|c|}{Maze with Dubins Car Dynamics and Continuous State and Action Spaces (No Training)} \\
 \hline
 \hline
 Maze Size  & 2 & 3 & 4 & 5 & 6  \\
 \hline
 \hline
AlphaZero&46.0$\pm$2.0&6.0$\pm$4.0&0.0$\pm$0.0&0.0$\pm$0.0&0.0$\pm$0.0\\
AlphaZero+CBE&29.0$\pm$3.0&9.0$\pm$4.0&1.0$\pm$1.0&1.0$\pm$1.0&1.0$\pm$1.0\\
Volume MCTS&43.0$\pm$5.0&\textbf{42.0$\pm$1.0}&\textbf{40.0$\pm$1.0}&\textbf{4.0$\pm$3.0}&\textbf{4.0$\pm$2.0}\\
OL AlphaZero&\textbf{49.0$\pm$0.0}&9.0$\pm$6.0&0.0$\pm$0.0&0.0$\pm$0.0&0.0$\pm$0.0\\
 \hline
\end{tabular}

\begin{tabular}{ |p{1.25cm}||p{0.9cm}|p{0.9cm}|p{0.9cm}|p{0.9cm}|p{0.9cm}|}
 \hline
 \multicolumn{6}{|c|}{Maze with Dubins Car Dynamics and Continuous State and Action Spaces (After Training)} \\
 \hline
\hline
 Maze Size & 2 & 3 & 4 & 5 & 6   \\
 \hline
 \hline
AlphaZero &37.0$\pm$4.0&8.0$\pm$5.0&0.0$\pm$0.0&0.0$\pm$0.0&0.0$\pm$0.0\\
AlphaZero+CBE &32.0$\pm$5.0&30.0$\pm$2.0&0.0$\pm$0.0&0.0$\pm$0.0&0.0$\pm$0.0\\
Volume MCTS &\textbf{49.0$\pm$0.0}&\textbf{46.0$\pm$0.0}&\textbf{38.0$\pm$1.0}&\textbf{34.0$\pm$1.0}&\textbf{26.0$\pm$3.0}\\
OL AlphaZero &\textbf{49.0$\pm$0.0}&0.0$\pm$0.0&0.0$\pm$0.0&0.0$\pm$0.0&0.0$\pm$0.0\\

 \hline
\end{tabular}
}
\vspace{-0.4cm}
\caption{\footnotesize Average rewards and standard errors on Dubins car environments. All methods use 5000 total rollouts per episode.}
\label{dubins_perf}
\vspace{-20pt}
\end{table}

To assess Volume-MCTS's performance and the impact of state occupancy measure regularization, we focus on robot navigation experiments. These are environments with significant practical interest where exploration is a central concern, and MCTS's exploration has historically been seen as insufficient by the robotics community.
We hypothesize that Volume-MCTS's state occupancy regularization term will motivate the planner to evenly explore the state space, resulting in better exploration. To test this, we conduct two sets of experiments. 
First, we use a 2D maze environment to visually compare exploration behavior of AlphaZero and Volume-MCTS. We perform an ablation study on this environment to test which factors are relevant to the algorithms' success. We then compare Volume-MCTS to an array of state-of-the-art methods on a challenging quadcopter navigation problem to test its performance in complex and realistic environments. 

\subsection{Maze Environments}
In these experiments, we hope to see how Volume-MCTS's Voronoi bias impacts its exploration behavior, and whether this bias can be matched by other common alterations to the AlphaZero algorithm. To address these questions, we compare Volume-MCTS's performance to three variants of AlphaZero: AlphaZero-Continuous \cite{Moerland2018}, an open-loop variant of AlphaZero-Continuous, and a variant of AlphaZero-Continuous with a CBE reward adapted from Never Give Up \cite{Badia2020NGU}. We use AlphaZero-Continuous because standard AlphaZero does not work on environments with continuous action spaces. AlphaZero-Continuous is a minimal extension using well-established techniques like progressive widening to generalize AlphaZero to the continuous-action setting. For the sake of brevity, we will refer to AlphaZero-Continuous simply as AlphaZero for all experiments. All environments use $\gamma = .95$ as a discount factor. 
 
 The maze environments we test on require long-horizon exploration that makes them inefficient to solve with standard RL methods, but which is well-suited to our method. Each environment is a maze with a continuous state and action space, where the agent must find a goal region opposite the starting location. The maze is $N$ tiles wide and tall, with random walls between these tiles. 
 We test two different sets of dynamics. The first is geometric dynamics, which are described by the update rule $s_{t+1} = s_t + a_t$. The second is the Dubins car dynamics, where the agent maneuvers a car with a fixed turning radius \cite{LaValle2006}. These dynamics can be highly challenging, as complex maneuvers may be needed to make sharp turns.

\begin{table*}[t]
{\tiny
\begin{tabular}{ |p{3cm}||p{1.2cm}|p{1.2cm}|p{1.2cm}|p{1.2cm}|p{1.2cm}|p{1.2cm}|p{1.2cm}|p{1.2cm}|  }
 \hline
 \multicolumn{9}{|c|}{Geometric Maze with Continuous State and Action Spaces (No Training)} \\
 \hline
 Maze Size  & 2 & 3 & 4 & 5 & 6 & 7 & 8 & 9  \\
 \hline
 \hline
Volume-MCTS&\textbf{49.0 $\pm$ 0.0}&\textbf{46.0 $\pm$ 0.0}&\textbf{43.0 $\pm$ 1.0}&\textbf{38.0 $\pm$ 1.0}&33.0 $\pm$ 1.0&31.0 $\pm$ 1.0&22.0 $\pm$ 4.0&7.0 $\pm$ 3.0\\
Volume-MCTS with no Rewards &47.7 $\pm$ 0.3&43.9 $\pm$ 0.1&40.0 $\pm$ 0.8&33.5 $\pm$  4.4&21.9 $\pm$ 15.2&29.5 $\pm$ 3.2&16.4 $\pm$ 13.4&9.1 $\pm$ 7.7\\
 \hline
 \end{tabular}
 } 
\caption{\footnotesize Comparison of Volume-MCTS with and without reward guidance. Statistically significant improvements are bolded}
\label{Ablation Study}
 \end{table*}

First, we visualize the search trees of AlphaZero and Volume-MCTS in space after 1000 expansions to see how they each perform at reaching novel areas.
\begin{figure}[h] 
\centering
\begin{subfigure}[t]{.43\linewidth}
\includegraphics[width=\linewidth]{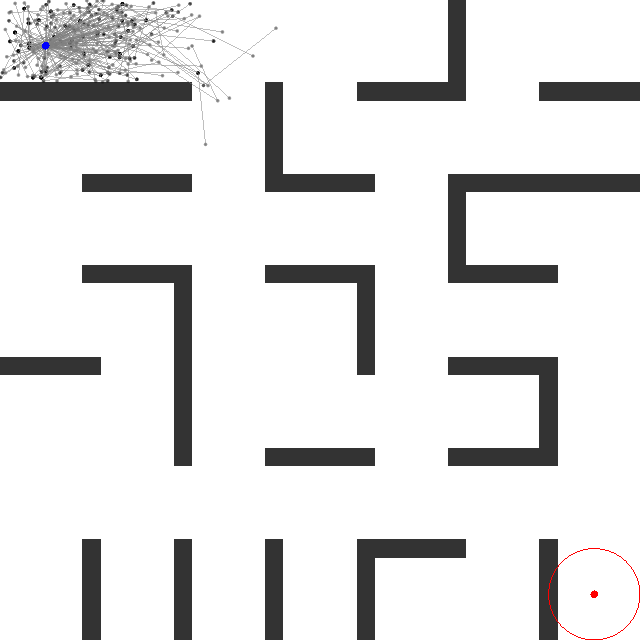}
\caption{\footnotesize AlphaZero}
\label{AZmaze}
\end{subfigure}
\hspace{10pt}
\begin{subfigure}[t]{.43\linewidth}
\includegraphics[width=\linewidth]{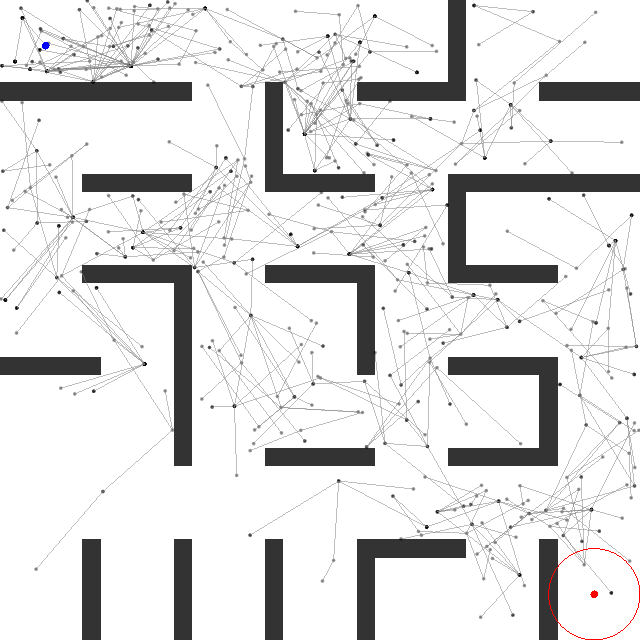}
\caption{\footnotesize Volume-MCTS}
\label{V-MCTSmaze}
\end{subfigure}

\caption{\footnotesize Comparison of AlphaZero and Volume-MCTS on the geometric maze environment}
 \label{exploration_comparison}
\vspace{-15pt}
\end{figure}


As expected, Volume-MCTS rapidly expands to cover the entire state space (Fig. \ref{V-MCTSmaze}), while AlphaZero stays very close to the starting location (Fig. \ref{AZmaze}). We argue that this is because AlphaZero does not distinguish states based on novelty. Instead, an $n$-action sequence that ends next to the starting location is counted as being just as novel as one that explores far from the start, as long as both have been expanded the same number of times. Additionally, progressive widening rules cause the tree to branch early -- child nodes will never have more branches than their parent. This gives progressive-widening-based approaches a high branching factor and short branches, which can make it difficult to find far-away goals.

We test Volume-MCTS’s exploration performance in comparison to the Continuous versions of AlphaZero, AlphaZero with CBE, and Open-loop AlphaZero on geometric mazes of increasing size (Table~\ref{geom_perf}). All methods use 5000 total rollouts for each episode. To evaluate the contributions of the planning algorithm separately from that of the neural network, we first test performance \textit{before} training the value or policy networks. While the average reward for Open-loop AlphaZero exceeds AlphaZero, we find that Volume-MCTS does far better, with the gap among the methods dramatically increasing for large environments with sparse rewards. Second, we compare performance after training. In general, all methods perform better with training than without 
As with the ‘before training’ results, the average reward for the AlphaZero methods falls to zero as maze size increases. Here, the differential in average reward for Volume-MCTS versus the other methods is far greater. Instead of falling to zero, the average reward for Volume-MCTS reaches a point where it stops decreasing altogether as the size of the environment increases. Thus, we find that Volume-MCTS performs significantly better than AlphaZero methods at maze navigation tasks. Further, it performs well even where AlphaZero with CBE struggles.

Next, we compare the same methods on a Dubins car maze, an environment with more challenging dynamics (Table~\ref{dubins_perf}). In the ‘before training’ results, the average reward for AlphaZero and Open-loop AlphaZero drops precipitously even for relatively small mazes. While AlphaZero with CBE does slightly better, Volume-MCTS results are significantly stronger. However, the differential among methods is much higher after training for 50,000 iterations. The average reward for the AlphaZero methods drops to zero for mazes of 4 by 4 tiles and larger 
\footnote{In our results, AlphaZero experienced stability issues. We suspect that this is due to bootstrapping problems common to many reinforcement learning methods. Volume-MCTS proved more stable in this regard because it develops deeper trees, which have a greater distance between the updated value and the target value.}. 
In contrast, Volume-MCTS experiences a far slower decline in average reward as maze size increases.



\subsection{Quadcopter Environment}

We  turn to the question of how Volume-MCTS performs on realistic problems, when compared to state-of-the-art planning and reinforcement learning methods. To assess this, we test its performance on a challenging quadcopter navigation task. In addition to AlphaZero \cite{Silver2018}, we compare against POLY-HOOT, a recent theoretically-sound MCTS algorithm for continuous action spaces \cite{POLY-HOOT},  Soft-Actor Critic with Hindsight Experience Replay (HER), a state-of-the-art model-free RL method for robotic tasks with sparse rewards \cite{HER}, and SST, a SBMP algorithm for kinodynamic systems \cite{Li_Littlefield_Bekris_2016}. 

\begin{figure}[t]
\centering
\begin{subfigure}[t]{.45\linewidth}
\centering
\includegraphics[width=\linewidth, valign=t]{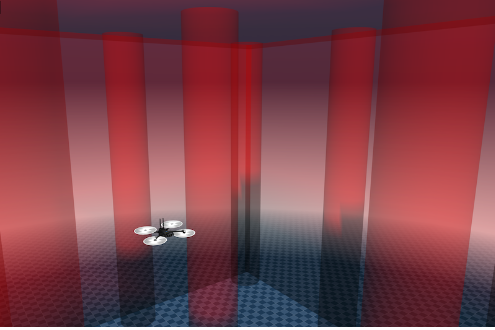}
\vspace{5pt}
\caption{\footnotesize Quadcopter environment}
 \label{Quadcopter render}
\end{subfigure}
\begin{subfigure}[t]{.52\linewidth}
\centering
\vspace{-19pt}
\includegraphics[width=\linewidth, valign=t]{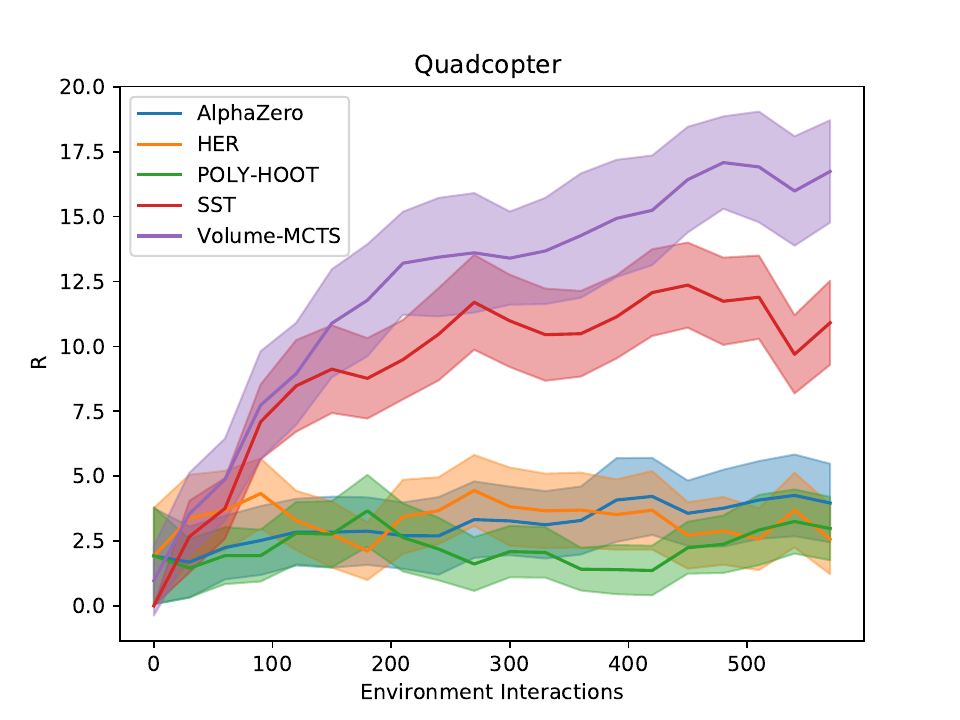}
\vspace{-5pt}
\caption{\footnotesize Reward as a function of total environmental interactions}
 \label{Quadcopter experiment}
\end{subfigure}
 \vspace{-10pt}
\end{figure} 




Figure \ref{Quadcopter experiment} shows that Volume-MCTS significantly outperforms HER and the MCTS algorithms. SST also performs well, which we find unsurprising due to its track record in robotics. Perhaps surprisingly, Volume-MCTS performs better than SST when longer search times are used. Additional experiments (included in Appendix \ref{Appendix: Additional Experiments}) show that SST and Volume-MCTS found the goal approximately the same fraction of the time, but that Volume-MCTS returns shorter paths on average. 
We believe this is primarily because the RL framing allows us to learn the value function while searching, enabling a better exploration/exploitation tradeoff. 

\subsection{Ablation Study}
To test the importance of using value to guide the search, we perform an ablation study. We test Volume-MCTS on the Maze environment against a variant which treats all rewards as 0. This variant is equivalent to Kinodynamic-RRT, because the probability of expanding a node depends only on that nodes's Voronoi region. Results are shown in Table \ref{Ablation Study}. 

We find that Volume-MCTS finds significantly shorter paths than the variant without reward. This indicates that using reward to guide the search allows us to make better exploration/exploitation tradeoffs than is possible for SBMP algorithms, which lack a notion of expected reward to guide the search. 



\section{Conclusion}
We show that two prominent exploration strategies, count-based exploration and Voronoi bias, can be seen as approximate solutions to a policy optimization objective with state occupancy measure regularization. Since both of these methods are established tools for long-horizon exploration in their respective fields, we argue that we can incentivize RL algorithms to explore effectively by minimizing 
this objective. 
While state occupancy objectives are typically non-convex in the policy and frequently intractable for RL methods, we show that this objective can be made convex and tractable on search trees. 
We use this insight to develop Volume-MCTS, a state occupancy-regularized planner that shows strong long-horizon exploration properties. 
We test our method on an array of robot navigation tasks, and find that Volume-MCTS outperforms methods from model-based RL, model-free RL, and SBMP. 
While Volume-MCTS is most applicable to deterministic navigation-focused environments, the connection we demonstrate between count-based exploration, Voronoi bias, and state-occupancy measure regularization is more general. We hope that additional work into state occupancy measure regularization will produce a powerful exploration objective that can be used across a wide variety of domains.

\clearpage


\bibliography{corl}  
\bibliographystyle{icml2021}

\newpage

\appendix

\onecolumn

\section{Algorithm Details} \label{Algorithm Details}

Here, we present the full pseudocode for the k-d tree algorithms described in the main paper. The KDTree\_value algorithm returns an estimate of the state's value in a given region by returning the average value of nodes near that state. KDBackprop updates the average value of nodes in the KDTree.

\begin{algorithm}[tbh]
  \caption{KDTree\_value}
  \label{alg:KDTreeValue}
\begin{algorithmic}
\STATE {\bfseries Have:} KDTree; 
\STATE {\bfseries Input:} state;
\STATE KDNode $\leftarrow$ KDTree.locate(state);
\STATE $i \leftarrow $ KDNode.depth
\WHILE{KDNode.depth > i / 2}
\STATE KDNode $\leftarrow$ KDNode.parent;
\ENDWHILE
\RETURN KDNode.value\_sum / KDNode.visit\_count
\end{algorithmic}
\end{algorithm}

\begin{algorithm}[tbh]
  \caption{KDBackprop}
  \label{alg:KDBackprop}
\begin{algorithmic}
\STATE {\bfseries Have:} KDTree; 
\STATE {\bfseries Input:} value, state;
\STATE KDTreeNode $\leftarrow$ KDTree.locate(state);
\STATE KDTreeNode.value\_sum $\leftarrow$ KDTreeNode.value\_sum + value;
\STATE KDTreeNode.visit\_count $\leftarrow$ KDTreeNode.visit\_count + 1;
\WHILE{KDTreeNode.has\_parent()}
\STATE KDTreeNode $\leftarrow$ KDTreeNode.parent;
\STATE KDTreeNode.value\_sum $\leftarrow$ KDTreeNode.value\_sum + value;
\STATE KDTreeNode.visit\_count $\leftarrow$ KDTreeNode.visit\_count + 1;
\ENDWHILE
\end{algorithmic}
\end{algorithm}

\begin{algorithm}[tbh]
  \caption{Expand}
  \label{alg:Expand}
\begin{algorithmic}
\STATE {\bfseries Have:} policy network $\pi_\theta$, value network $V_\theta$;
\STATE {\bfseries Input:} node $n$;
\STATE $a \sim \pi_\theta(n$.state);
\STATE $s' \leftarrow T(n$.state, $a$);
\STATE kd\_node $\leftarrow$ KDTree.add($V_\theta(s')$)
\STATE $\hat{v} \leftarrow $KDTree\_value($s'$)
\STATE $n$.children $\leftarrow n$.children $\cup$ MCTSNode($a, s'$, kd\_node, $\hat{v}$)
\RETURN $\hat{v}$
\end{algorithmic}
\end{algorithm}


Earlier, we said that we conjectured that the KDTree value estimate was a consistent estimator. There are two main challenges to proving this. First, the estimator builds estimates from MCTS node values, which are not neither stationary nor independent, which makes traditional bias and variance analysis difficult. Secondly, it is difficult to obtain guarantees on the shape and regularity of the KD regions. It’s necessary that, with high probability, the diameter of a KD region goes to zero as the volume does. This is clearly true in practice but tedious to show mathematically. If both of these issues are dealt with or assumed away, the proof outline in the text can be easily formalized.

\section{Extension to Non-deterministic Environments and Action-dependent Rewards} \label{Extension}

To account for action-based rewards, we regularize the policy in addition to the state-action occupancy measure, giving use the loss $\mathcal{L} = E_{n \sim d_{\pi}}[\mathcal{V}(n) - \lambda D_f(\pi(n) || \pi_0(n))] - \lambda D_f(\phi(d_{\pi}) || \phi_0)$, where $\pi_0$ is the uniform distribution over tree moves. Repeating the derivation of Prop 1 and Thm 2 gives us the tree policy $$\pi^*(a \mid n) = \frac{\lambda (SubtreeVol(a) + \frac{\gamma^{d(n)} P(\text{$n$ reached } \mid \pi)}{1 + len(n.children)})}{\alpha - \gamma^{d(n)} P(\text{$n$ reached } \mid \pi) Q^{\pi^*}(a \mid n)}$$ 
This is effectively the same method, but with an added offset to each action's volume. We implemented it and found that this regularization was detrimental to long-horizon exploration, so we did not use it in our experiments. However, it is easy to implement when rewards are action-dependent. 

In principle, Volume-MCTS could also be extended to non-deterministic environments using double-progressive widening ~\citep{bertsimas2014doubleprogressivewidening}. In double progressive widening, the tree keeps a list of sampled "next states" for each action, and expands this list over time using progressive widening. This allows the tree to estimate the value under stochastic dynamics. It is straightforward to implement this approach with Volume-MCTS by making two simple changes to the stated algorithm. 
\begin{enumerate}
    \item $\Vol(a) = \sum_i \Vol(s_i')$, where $s_i' \in next(a)$ is the list of next states from $a$
    \item $Q(a) = \sum_i \frac{\V(s_i')}{|next(a)|}$
\end{enumerate}

For each of these alterations, we chose not to include them in the main algorithm because they introduce additional complexity to the implementation, and also because they may impact the exploration efficiency. Analysis of the exploration efficiency of these variants is a promising area for future work. 


\section{Proofs} \label{Appendix: Proofs}
\subsection{Background review}

\subsubsection{$f$-divergences}
\begin{define}
    For any probability distributions $p$ and $q$ on $\mathcal{X}$ and function $f \colon \mathbb{R} \to \mathbb{R}$ such that $f$ is a convex
    function on $\mathbb{R}$ and $f(1) = 0$, the $f$-divergence $D_f$ between $p$ and $q$ is defined as
    \[
        D_f(p \mid \mid q) = \int_\mathcal{X} q(x) f \left( \frac{p(x)}{q(x)} \right) dx
    \]        
\end{define}

\subsubsection{Convex functions of probability distributions}

Convex optimization on probability distributions may be seen as a constrained optimization problem, as the probability distribution is constrained to sum to exactly 1. Under the KKT conditions (which the problems we are considering meet), there is guaranteed to be a unique solution such that $\frac{\partial \mathcal{L}}{\partial \sr(n)} = \alpha$ for all $n$. 

\subsection{Proofs}
\
\subsubsection{Uniqueness of $\srstar$} \label{Uniqueness Proof}
\begin{propn}[\ref{prop:prop1}]

    Let $\V(n) = E[(\sum_{i=0}^{T-1} \gamma^i R(\traj{n}_i)) + \gamma^T V^\pi_\theta(\traj{n}_T)]$. \\
    Let $\mathcal{L}(\sr) = E_{n \sim \sr}[\V(n)] - \lambda D_f(\SR(T, \sr) \mid \mid \SR_0)$
    
    Suppose $f(t) = - \ln{t}$ and $\SR(T, \sr)(s)$ is a partition density estimator, where the associated measure of any node $n$ is $\Vol(n)$. 
    Then $\mathcal{L}$ has a unique optimizer $\srstar$, such that $\srstar(n) =  \frac{\lambda}{\alpha - \V(n)} \Vol(n)$, where $\alpha$ is a constant that makes $\srstar$ a proper probability distribution.
\end{propn}

\begin{proof}
    Observe that $\mathcal{L}(\sr)$ is a convex function. Since $\sr$ is a probability distribution, it is known that there exists a unique maximizer for $\mathcal{L}$, which occurs when $\sr$ is normalized and $\frac{\partial}{\partial \sr(n)} \mathcal{L}(\sr) = \alpha$, where $\alpha$ is a constant. We solve for this value now. 

\begin{align*}
     \alpha =& \frac{\partial}{\partial \sr(n)} \mathcal{L}(\sr) \\
     =& \frac{\partial}{\partial \sr(n)} E_{n \sim \sr}[\mathcal{V}(n)]  - \lambda D_f(\SR(T, \sr) \mid \mid \SR_0) \\
     =& \frac{\partial}{\partial \sr(n)} \sum_n \mathcal{V}(n)\sr(n)  - \lambda \int_S \SR_0(s) f\left(\frac{\SR(T, \sr)(s)}{\SR_0(s)}\right) ds  \\
     =&  \mathcal{V}(n)  - \lambda \int_S \frac{\partial}{\partial \sr(n)} \SR_0(s) f\left(\frac{\SR(T, \sr)(s)}{\SR_0(s)}\right) ds  \\
     =&  \mathcal{V}(n)  + \lambda \int_S \frac{\partial}{\partial \sr(n)} \SR_0(s) \log\left(\frac{\SR(T, \sr)(s)}{\SR_0(s)}\right) ds  \\
     =&  \mathcal{V}(n)  + \lambda \int_S \SR_0(s) \frac{\partial}{\partial \sr(n)}  \log\left(\frac{\SR(T, \sr)(s)}{\SR_0(s)}\right) ds  \\
     =&  \mathcal{V}(n)  + \lambda \int_S \SR_0(s)  \frac{\SR_0(s)}{\SR(T, \sr)(s)} \frac{\partial}{\partial \sr(n)} \frac{\SR(T, \sr)(s)}{\SR_0(s)} ds  \\
     =&  \mathcal{V}(n)  + \lambda \int_S \frac{\SR_0(s)}{\SR(T, \sr)(s)} \frac{\partial}{\partial \sr(n)} \SR(T, \sr)(s) ds.
\end{align*}

Observe that by assumption, $\SR$ is a partition density estimator with weights $\sr$. Therefore, $\frac{\partial}{\partial \sr(n)} \SR(T, \sr)(s)$ is $\frac{\partial}{\partial \sr(n)} \sr(n) g(T, s) = g(T, s)$ for some function $g$ if $s$ is in $n$'s associated region and 0 otherwise. Let $\mathds{1}(n, s)$ be 1 if $s$ is in $n$'s associated region $R(n)$, and 0 otherwise. Then, 
\begin{align*}
     \alpha =& \frac{\partial}{\partial \sr(n)} \mathcal{L}(\sr) \\
     =&  \mathcal{V}(n)  + \lambda \int_S \frac{\SR_0(s)}{\SR(T, \sr)(s)} \frac{\partial}{\partial \sr(n)} \SR(T, \sr)(s) ds \\
     =&  \mathcal{V}(n)  + \lambda \int_S \frac{\SR_0(s)}{\SR(T, \sr)(s)} g(T, s) \mathds{1}(n, s) ds \\
     =&  \mathcal{V}(n)  + \lambda \int_{R(n)} \frac{\SR_0(s)}{\SR(T, \sr)(s)} g(T, s) ds \\
     =&  \mathcal{V}(n)  + \lambda \int_{R(n)} \frac{\SR_0(s)}{\sr(n) g(T, s)} g(T, s) ds \\
     =&  \mathcal{V}(n)  + \lambda \int_{R(n)} \frac{\SR_0(s)}{\sr(n)} ds \\
     =&  \mathcal{V}(n)  + \lambda \frac{1}{\sr(n)} \int_{R(n)} \SR_0(s) ds \\
     =&  \mathcal{V}(n)  + \lambda \frac{1}{\sr(n)} Vol(n) \\
     \alpha - \mathcal{V}(n) =&  \frac{ \lambda Vol(n)}{\sr(n)} \\
     \sr(n) =&  \frac{ \lambda Vol(n)}{\alpha - \mathcal{V}(n)}.
\end{align*}
\end{proof}

\subsubsection{$\srstar$ can be found by optimizing $\pi$}  \label{Convexity Proof}


We begin by proving that on trees, there is a one-to-one correspondence between $\sr$ and $\pi$. This is not true for general MDPs, because actions can be redundant such that multiple actions can lead to the same state. This one-to-one correspondence will allow us to show that the unique solution to the policy optimization problem is also the unique solution to the state-occupancy measure optimization problem. This will allow us to sample from $\sr$ by calculating $\pi$ and traversing the tree, which is more efficient than computing and sampling from $\sr$ directly. 

\begin{lem}\label{lem:lemma1}
    For any selection of $\pi$, $\pi$ uniquely determines
    \[
    \sr(n) = \pi(\stay \mid n) \prod_{a_i > n} \pi(a_i \mid n_i).
    \]
\end{lem}
\begin{proof}
    Suppose the algorithm starts at the root node of the tree and selects moves until it selects a "stay" move and stops. Observe that there is exactly one sequence of tree moves that leads to the algorithm stopping at $n$. 
    When traversing the tree, the algorithm must first select the action that is the ancestor of $n$ at each step. Then it must choose to stop at $n$. Since these moves are independent, we can write the probability of the algorithm stopping at $n$ as 
    \[
    \sr(n) = \pi(\stay \mid n) \prod_{a_i > n} \pi(a_i \mid \parent(a_i)).
    \]
\end{proof}




\begin{lem}\label{lem:lemma2}
    Exactly one policy $\pi$ produces a given node distribution $\sr(n)$. This policy is 
    \[
    \pi(a \mid n) = \frac{\sum_{n_i < a} \sr(n_i)}{\prod_{a_i > n} \pi(a_i \mid \parent(a_i))}.
    \]
\end{lem}
\begin{proof} 
    Consider that the probability of the algorithm reaching a node when traversing the tree $P(n)$. By the same argument used in the previous lemma, this probability is simply the product of all ancestor probabilities, $P(n) = \prod_{a_i > n} \pi(a_i \mid \parent(a_i))$. Also note that some action in the tree $a$ can be taken only if its parent has been reached. Similarly, if action $a$ is taken, then the algorithm is guaranteed to reach $a$'s child. Therefore, $P(\child(a)) = P(a) = \pi(a \mid n) P(n) = \pi(a \mid n) \prod_{a_i > n} \pi(a_i \mid \parent(a_i))$. 
    \par
    Observe that once the algorithm enters a subtree, it has no way to return, so it is guaranteed to stop at one of the nodes in the subtree. Since the algorithm stops in the subtree of $n$ if and only if $n$ is reached, the sum of the stopping probabilities in a subtree is exactly equal to the probability that $n$ is reached.  Hence, $P(n) = \sum_{n_i < a} \sr(n_i)$.
    Therefore, 
    \begin{align*}
        \pi(a \mid n) &= \pi(a \mid n) \frac{P(n)}{P(n)} \\
            &= \frac{\pi(a \mid n) P(n)}{P(n)} \\
            &= \frac{P(a)}{P(n)} \\
            &= \frac{P(\child(a))}{P(n)} \\
            &= \frac{\sum_{n_i < a} \sr(n_i)}{\prod_{a_i > n} \pi(a_i \mid \parent(a_i))}.
    \end{align*}
\end{proof}





Lemma~\ref{lem:lemma2} is important because for non-tree environments, multiple policies can produce the same state occupancy. As we proceed, we will use tools from convex optimization to find the optimal policy, which requires us to know that the optimal policy is unique. 


\begin{lem}\label{lem:lemma3}
    \[
    \frac{\partial \sr(n')}{\partial \pi(a | n)} = \mathds{1}(n' < a) \frac{\sr(n')}{\pi(a | n)}.
    \]
\end{lem}
\begin{proof}
    Observe that $\frac{\partial \sr(n')}{\partial \pi(a | n)} = 0$ if $n'$ is not a descendent of $a$, because $\sr(n')$ does not depend on $\pi(a \mid n)$. 
    If $n'$ is a descendent of $a$, then, 
    \begin{align*}
        \frac{\partial}{\partial \pi(a | n)} \sr(n') &= \frac{\partial}{\partial \pi(a | n)} (\pi(\stay | n') \prod_{\text{action } a' > n'} \pi(a' | \parent(a') ) \\
        &= (\pi(\stay | n') \frac{\partial}{\partial \pi(a | n)}  \prod_{\text{action } a' > n'} \pi(a' | \parent(a') ) \\
        &= \pi(\stay | n') \frac{\partial}{\partial \pi(a | n)} pi(a | n)  \prod_{\text{action } a' > n', a' \neq a} \pi(a' | \parent(a')) \\
        &= \pi(\stay | n')  \prod_{\text{action } a' > n', a' \neq a} \pi(a' | \parent(a')) \\
        &= \pi(\stay | n')  \frac{\prod_{\text{action } a' > n'} \pi(a' | \parent(a')}{\pi(a | n)} \\
        &= \frac{\sr(n')}{\pi(a | n)}.
    \end{align*}
\end{proof}




\begin{lem}\label{lem:lemma4}
    For any function $f$, 
    \begin{align*}
        \frac{\partial f}{\partial \pi(a \mid n_t)} & = P(n_t) E_{n' \sim \sr(\cdot \mid n_t, a)}\left[\frac{\partial f}{\partial \sr(n')}\right].
    \end{align*}
\end{lem}
\begin{proof}
    \begin{align*}
	    \frac{\partial f}{\partial \pi(a | n_t)} &= \sum_{n'} \frac{\partial f}{\partial \sr(n')} \frac{\partial \sr(n')}{\partial \pi(a | n_t)} \\
    \end{align*}
    Observe that $\frac{\partial \sr(n')}{\partial \pi(a | n_t)} = 0$ if  $n'$ is not a descendant of $a$.  
    \begin{align*}			
		\frac{\partial f}{\partial \pi(a | n_t)} &= \sum_{n' < a} \frac{\partial f }{\partial \sr(n')} \frac{\partial \sr(n')}{\partial \pi(a | n_t)} \\
        &= \sum_{n' < a} \frac{\partial f }{\partial \sr(n')} \frac{\sr(n')}{\pi(a | n_t)} \\
        &= \frac{1}{\pi(a | n_t)} \sum_{n' < a} \frac{\partial f }{\partial \sr(n')} \sr(n') \\
        &= \frac{P(n_t)}{\pi(a | n_t)P(n_t)} \sum_{n' < a} \frac{\partial f }{\partial \sr(n')}  \sr(n') \\
    \end{align*}
    Observe that $\pi(a | n_t)P(n_t) = P(n_{t+1})$, where $n_{t+1}$ is the child node of $a$. 
    \begin{align*}		
        \frac{\partial f}{\partial \pi(a | n_t)} &=  \frac{P(n_t)}{P(n_{t+1})} \sum_{n' < a} \frac{\partial f }{\partial \sr(n')}  \sr(n') \\
        &=  \frac{P(n_t)}{P(n_{t+1})} \sum_{n' \leq n_{t+1}} \frac{\partial f }{\partial \sr(n')}  \sr(n') \\
        &=  P(n_t) \sum_{n' \leq n_{t+1}} \frac{\partial f }{\partial \sr(n')} \sr(n') (P(n_{t+1})) \\
        &=  P(n_t) \sum_{n' \leq n_{t+1}} \frac{\partial f }{\partial \sr(n')} \sr(n' | n_{t+1}) \\
        &=  P(n_t) E_{n' \sim  \sr(n' | n_{t+1})}\left[\frac{\partial f }{\partial \sr(n')}\right] \\
        &=  P(n_t) E_{n' \sim  \sr(n' | n_t, a)}\left[\frac{\partial f }{\partial \sr(n')}\right].
    \end{align*}
\end{proof}

\begin{thmn}[\ref{thm:theorem1}]
    Let $T$ be a tree. Let $\pi(\cdot \mid n)$ be a probability distribution over moves that may be taken from node $n$, where $n$ is a node in the tree. Let $\sr$ be the probability of stopping at any given node if we traverse the tree by sampling moves from $\pi$. 

    Then for any convex loss function $\mathcal{L}(\sr)$, $\mathcal{L}$ is convex with respect to $\pi(a \mid n)$ for all nodes $n$ and moves $a$. Furthermore, $\mathcal{L}$ is minimized if and only if for every node $n$, $\frac{\partial \mathcal{L}}{\partial \pi(a \mid n)}$ is constant for all moves $a$.
\end{thmn}
\begin{proof}
    The first portion of the proof is easy to establish. Suppose $\mathcal{L}(\sr)$ is convex. Recall that $\sr(n) = \pi(\stay \mid n) \prod_{a_i > n} \pi(a_i \mid \parent(a_i))$. Observe that $\sr(n)$ is linear with respect to the probability of each action $a_i$ above $n$ in the tree. Since a convex function of a linear function is still convex, $\mathcal{L}(\pi(a \mid n))$ is convex. 
    \par
    However, this alone does not establish that optimizing the policy at each node is sufficient to optimize the loss. While $\mathcal{L}$ is convex in $\pi(a \mid n)$ for all $a, n$, this does not necessarily imply that $\mathcal{L}$ is jointly convex in $\pi(a \mid n)$ and $\pi(a' \mid n')$ for $a \neq a', n \neq n'$. \citet{Hazan2019} provide a counterexample to this, where two distinct policies each induce a uniform distribution over a set of states, but a linear combination of the policies is non-uniform. This implies that the entropy of the state occupancy measure is not convex with respect to the policy. However, this counterexample relies on a directed acyclic graph structure. We aim to show that for trees, convex functions of $\sr$ are minimized if and only if a function is minimized with respect to the policy at every node. While this does not necessarily imply that the loss is convex in $\pi$, it does imply that we can minimize convex functions by doing convex optimization of the policy at each node. 
    We prove this by contradiction. 
    \par
    Recall that for a convex loss, $\sr$ is a unique global optimum of $\mathcal{L}(\sr)$ if and only if $\frac{\partial \mathcal{L}}{\partial \sr(n)} = \alpha$ for all $n$. Suppose this condition holds for our selection of $\sr$. Then by Lemma~\ref{lem:lemma3}, $\frac{\partial \mathcal{L}}{\partial \pi(a | n_t)} = P(n_t) E_{n' \sim  \sr(\cdot | n_t, a)}[\frac{\partial \mathcal{L} }{\partial \sr(n')}] = P(n_t) E_{n' \sim  \sr(\cdot | n_t, a)}[\alpha] = P(n_t) \alpha$. Observe that $\frac{\partial \mathcal{L}}{\partial \pi(a | n_t)}$ does not depend on $a$. Therefore $\frac{\partial \mathcal{L}}{\partial \pi(a | n_t)}$ is constant for all moves $a$. 
    \par
    Suppose that for every node $n$, $\frac{\partial \mathcal{L}}{\partial \pi(a | n)}$ is constant for all moves $a$. We use proof by contradiction. 
    Let $\srstar$ be the optimal value for $\sr$. Suppose $\sr \neq \srstar$. Then there exists some $n$ where $\frac{\partial \mathcal{L}}{\partial \sr(n)} > \alpha$ or $\frac{\partial \mathcal{L}}{\partial \sr(n)} < \alpha$. We consider the first case.  Observe that since $\mathcal{L}$ is strongly convex with respect to $\sr$, $\frac{\partial \mathcal{L}}{\partial \sr(n)}$ is monotonically increasing with $\sr(n)$. Hence, if $\frac{\partial \mathcal{L}}{\partial \sr(n)} > \alpha$, $\sr(n) > \srstar(n)$. Since $\sr$ sums to 1, there must also be some $\Tilde{n}$ where $\sr(\Tilde{n}) < \srstar(\Tilde{n})$. 
    \par
    By Lemma~\ref{lem:lemma4}, $\frac{\partial \mathcal{L}}{\partial \pi(a | n)}  = P(n) E[\frac{\partial \mathcal{L}}{\partial \sr(n)}]$ for all $n$. 
    Since $\frac{\partial \mathcal{L}}{\partial \sr(n)} > \alpha$,  $\frac{\frac{\partial \mathcal{L}}{\partial \pi(\stay | n)} }{P(n)} = E_{n' \mid \stay, n}[\frac{\partial \mathcal{L}}{\partial \sr(n')}] = \frac{\partial \mathcal{L}}{\partial \sr(n)}  > \alpha$. 
    Since we assumed that $\frac{\partial \mathcal{L}}{\partial \pi(a | n)} $ was constant for all moves from $n$, $\frac{\frac{\partial \mathcal{L}}{\partial \pi(a | n)} }{P(n)} > \alpha$ for all moves. 
    Now, consider $n$'s parent action $a_1$. Since $E_{n' \sim  \sr(\cdot | n)}[\frac{\partial \mathcal{L}}{\partial \sr(n')}] > \alpha$ for each branch of n, $E_{n' \sim  \sr(\cdot | n)}[\frac{\partial \mathcal{L}}{\partial \sr(n')}] > \alpha$ for the whole subtree of $a_1$. 
    Therefore, $\frac{\frac{\partial \mathcal{L}}{\partial \pi(a_1 | \text{ parent}(a_1))} }{P(\text{ parent}(a_1))} = E_{n' \sim  \sr(\cdot | \text{ parent}(a_1))}[\frac{\partial \mathcal{L}}{\partial \sr(n')}] > \alpha$. 
    Here, the same reasoning applies as before -- the gradient is equal for all moves, so they must all have a gradient > $P(\text{ parent}(a_1)) \alpha$. We can repeat this argument with induction for each node in the tree, showing that each ancestor must have a gradient greater than alpha, until we reach the root. Therefore, the gradient of all moves at the root is greater than $\alpha$. 
    \par
    However, recall that there must be a node $\Tilde{n}$ where $\sr(\Tilde{n}) < \srstar(\Tilde{n})$. Since $\frac{\partial \mathcal{L}}{\partial \sr(\Tilde{n})}$ is monotonically increasing with $\sr(\Tilde{n})$, $\frac{\partial \mathcal{L}}{\partial \sr(\Tilde{n})} < \alpha$. We make the same argument as before, showing that $\frac{\frac{\partial \mathcal{L}}{\partial \pi(\Tilde{a} | \Tilde{n})} }{P(\Tilde{n})} < \alpha$ for all moves from $\Tilde{n}$. By a symmetrical argument to the previous paragraph, all of $\Tilde{n}$'s ancestors must have gradients < $\alpha$. Therefore the gradient at the root is < $\alpha$. 
    \par
    However, we already showed that the gradient at the root was > $\alpha$, so this is a contradiction. Therefore our assumption must be false, and $\frac{\partial \mathcal{L}}{\partial \sr(n)}= \alpha$ for all nodes. Hence, the induced distribution of $\pi^*$ is $\srstar$. 
    \par
    The same contradiction is reached if we assume that there exists some $n$ where $\frac{\partial \mathcal{L}}{\partial \sr(n)} < \alpha$. 
\end{proof}

\begin{thmn}[\ref{thm:theorem2}]
     Suppose $f(t) = - \ln{t}$ and $\SR(T, \sr)(s)$ is a partition density estimator. \\
    Then $\mathcal{L}(\pi)$ has a unique optimizer $\pi^*$, such that  $\pi^*(a \mid n) = \frac{\lambda \SubtreeVol(a)}{\alpha - \gamma^{d(n)} P(\text{$n$ reached} \mid \pi) Q^{\pi^*}(a \mid n)}$, where --
    $\SubtreeVol(a)$ is the total volume of all nodes in the subtree of $a$. $d(n)$ is the depth of $n$ in the search tree. 
    $P(\text{$n$ reached} \mid \pi)$ is the probability that the algorithm reaches $n$ when traversing the tree, or equivalently that the algorithm stops at some node in $n$'s subtree. 
    and $\alpha$ is whatever constant normalizes $\pi^*$.  
    
    Additionally, $\pi^*$ is the unique distribution that induces $\srstar$ as the state occupancy measure.
\end{thmn}

\begin{proof}
As before, $\mathcal{L}$ is convex with respect to $\pi(a \mid n)$, so it has a unique optimizer when $\frac{\partial \mathcal{L}}{\partial \pi(a \mid n)}$ is equal to a constant for all $a$.
Recall that by Lemma~\ref{lem:lemma4}, $ \frac{\partial \mathcal{L}}{\partial \pi(a \mid n)} = P(n) E_{n' \sim \sr(\cdot \mid n, a)}[\frac{\partial \mathcal{L}}{\partial \sr(n')}]$. 
Then 
\begin{align*}
    \frac{\partial \mathcal{L}}{\partial \pi(a \mid n)} &= P(n) E_{n' \sim \sr(\cdot \mid n, a)}\left[\frac{\partial \mathcal{L}}{\partial \sr(n')}\right] \\ 
    &= P(n) E_{n' \sim \sr(\cdot \mid n, a)}\left[\mathcal{V}(n') + \lambda \frac{\Vol(n')}{\sr(n')}\right] \\ 
    &= P(n) E_{n' \sim \sr(\cdot \mid n, a)}[\mathcal{V}(n')] + \lambda P(n) E_{n' \sim \sr(\cdot \mid n, a)}\left[\frac{\Vol(n')}{\sr(n')}\right] \\     
    &= P(n) E_{n' \sim \sr(\cdot \mid n, a)}[\mathcal{V}(n')] + \lambda P(n) \sum_{n' < a} \frac{\sr(n')}{P(n) \pi(a \mid n)}\left[\frac{\Vol(n')}{\sr(n')}\right] \\ 
    &= P(n) E_{n' \sim \sr(\cdot \mid n, a)}[\mathcal{V}(n')] + \lambda \sum_{n' < a} \frac{\Vol(n')}{\pi(a \mid n)} \\ 
    &= P(n) \left(\gamma^d Q^\pi(n, a) + \sum_{i=0}^d \gamma^i R(n_i, a_i)\right) + \lambda \frac{\SubtreeVol(a)}{\pi(a \mid n)} \\ 
    &= P(n) \left(\sum_{i=0}^d \gamma^i R(n_i, a_i)\right) +  P(n) (\gamma^d Q^\pi(n, a)) + \lambda \frac{\SubtreeVol(a)}{\pi(a \mid n)} \\ 
\end{align*}
Observe that since $P(n) (\sum_{i=0}^d \gamma^i R(n_i, a_i))$ is constant for all moves, we may absorb it into $\alpha$
\begin{align*}
    \alpha &=  P(n) \gamma^d Q^\pi(n, a) + \lambda \frac{\SubtreeVol(a)}{\pi(a \mid n)} \\ 
    \alpha - P(n) \gamma^d Q^\pi(n, a) &= \lambda \frac{\SubtreeVol(a)}{\pi(a \mid n)} \\ 
    \pi(a \mid n) &= \lambda \frac{\SubtreeVol(a)}{\alpha - P(n) \gamma^d Q^\pi(n, a)} \\ 
\end{align*}

By Theorem 2, $\pi(a \mid n) = \lambda \frac{\SubtreeVol(a)}{\alpha - P(n) \gamma^d Q^\pi(n, a)}$ is the unique optimizer of $\mathcal{L}$. 
\end{proof}

\subsubsection{Connection to Count-Based Exploration} \label{CBE Proof}
\begin{propn} [\ref{prop:prop2}]
Suppose $D_f$ is chosen to be the Hellinger distance, $f(t) = 2(1-\sqrt{t})$, and $\SRhat$ is chosen to be kernel density estimator, $\SRhat((T, \srhat))(s) = \sum_{i \in T} \srhat(i) k(\state{i}, s)$. Additionally, suppose $\SRprior$ is the uniform distribution over the state space. 
Let $R_{CBE}(n) = \sqrt{\frac{1}{\sum_{i \in T} k(\state{i}, \state{n})}}$, the count-based exploration reward described in ~\cite{Badia2020NGU}. 
Then,
    \begin{align*}
        a   &= \argmax_n \frac{\partial}{\partial \srhat(n)} E_{n' \sim \srhat}[\V(N')]  - \lambda D_f(\SRprior \mid \mid \SRhat) \\
            &\approx \argmax_a  Q(s, a)  + c E_{n' \sim \subtree(a)}\left[ R_{CBE}(n') \right] 
    \end{align*}
\end{propn}

\begin{proof}
    We aim to show that applying the empirical decision rule to the policy optimization problem with Hellinger squared distance regularization over the state spaces yields a count-based exploration reward. 
    The squared Hellinger distance is an $f$-divergence, where $f(t) = 2(1-\sqrt{t})$. Observe that the derivative of $f$, $\frac{df}{dt}$, is $\frac{df}{dt}(t) = -\frac{1}{\sqrt{t}}$.

\begin{align*}
    \frac{\partial}{\partial \srhat(n)}  \mathcal{L}(\srhat) 
    &=  \frac{\partial}{\partial \srhat(n)} E_{n' \sim \srhat}[\V(n')]  - \lambda D_f(\SRprior \mid \mid \SRhat) \\
    &=  \frac{\partial}{\partial \srhat(n)} \sum_{n'} \srhat(n') \V(n')   - \lambda \int_S \SRprior(s) f\left(\frac{\SRhat(s)}{\SRprior(s)}\right) ds \\
    &=  \V(n)  - \lambda \int_S \SRprior(s) \frac{\partial}{\partial \srhat(n)}  f\left(\frac{\SRhat(s)}{\SRprior(s)}\right) ds \\
    &=  \V(n)  - \lambda \int_S \SRprior(s)  \frac{df}{dt}\left(\frac{\SRhat(s)}{\SRprior(s)}\right) \frac{\partial}{\partial \srhat(n)} \frac{\SRhat(s)}{\SRprior(s)} ds \\
    &=  \V(n)  - \lambda \int_S \SRprior(s)  \frac{df}{dt}\left(\frac{\SRhat(s)}{\SRprior(s)}\right)  \frac{\partial}{\partial \srhat(n)} \frac{\sum_{i \in T} \srhat(i) k(\state{i}, s)}{\SRprior(s)} ds \\
    &=  \V(n)  - \lambda \int_S \SRprior(s)  \frac{df}{dt}\left(\frac{\SRhat(s)}{\SRprior(s)}\right)  \frac{k(\state{n}, s)}{\SRprior(s)} ds \\
    &=  \V(n)  - \lambda \int_S \SRprior(s)    \frac{df}{dt}\left(\frac{\sum_{i \in T} \srhat(i) k(\state{i}, s)}{\SRprior(s)}\right) \frac{k(\state{n}, s)}{\SRprior(s)} ds \\
    &=  \V(n)  - \lambda \int_S \frac{df}{dt}\left(\frac{\sum_{i \in T} \srhat(i) k(\state{i}, s)}{\SRprior(s)}\right) k(\state{n}, s) ds
\end{align*}

Recall that $f(t) = 2(1-\sqrt{t})$ and $\frac{df}{dt}(t) = -\frac{1}{\sqrt{t}}$. 

\begin{align*}
   \frac{\partial}{\partial \srhat(n)}  \mathcal{L}(\srhat) 
   &=  \V(n)  - \lambda \int_S \frac{df}{dt}\left(\frac{\sum_{i \in T} \srhat(i) k(\state{i}, s)}{\SRprior(s)}\right)  k(\state{n}, s) ds \\
   &=  \V(n)  + \lambda \int_S \sqrt{\frac{\SRprior(s)}{\sum_{i \in T} \srhat(i) k(\state{i}, s)}}  k(\state{n}, s) ds 
\end{align*}
 
Now, we observe two properties of kernels assumed for kernel regression. First, kernels are window functions: a kernel $k(x, y)$ is maximal when $x=y$ and decreases rapidly as $\mid x - y \mid$ becomes large. Additionally  $\int_\mathcal{X} k(x, y) dx = 1$ for $x, y \in \mathcal{X}$. This means that when integrating another function with the kernel, almost all the contribution comes from near the center. This means that for a continuous function $f$, we may approximate the integral $\int_\mathcal{X} f(x) k(x, y) dx$ by linearly approximating $f$ about $x=y$ where the kernel is maximized. Hence $\int_\mathcal{X} f(x) k(x, y) dx \approx \int_\mathcal{X} [f(y) + (x-y)\cdot  \nabla f(y)] k(x, y) dx$. And second, we assume that the kernel is an even function, so integrating $\int_\mathcal{X} (x-y) k(x, y) dx = 0$. Hence, $\int_\mathcal{X} f(y) + (x-y)\cdot \nabla f(y)] k(x, y) dx = \int_\mathcal{X} f(y) k(x, y) dx+ \int_\mathcal{X} (x-y)\cdot \nabla f(y) k(x, y) dx = f(y) + 0 = f(y)$. Applying this to the derivation from above, we find the following: 

\begin{align*}
   \frac{\partial}{\partial \srhat(n)}  \mathcal{L}(\srhat) &=  \V(n)  + \lambda \int_S k(\state{n}, s) \sqrt{\frac{\SRprior(s)}{\sum_{i \in T} \srhat(i)  k(\state{i}, s)}} ds \\
   &\approx \V(n)  + \lambda \sqrt{\frac{\SRprior(\state{n})}{\sum_{i \in T} \srhat(i) k(\state{i}, \state{n})}}
\end{align*}

$\srhat$ was defined to be $\frac{1}{N}$ for all nodes. $\lambda$ was assumed to be $\frac{c}{\sqrt{N}}$ for some constant $c$. Combining constants and simplifying, this yields

\begin{align*}
    \frac{\partial}{\partial \srhat(n)}  \mathcal{L}(\srhat) 
    &\approx \V(n)  + \lambda \sqrt{\frac{\SRprior(\state{n})}{\sum_{i \in T} \srhat(i) k(\state{i}, \state{n})}} \\
   &= \V(n)  + c\frac{1}{\sqrt{N}} \sqrt{\frac{\SRprior(\state{n})}{\sum_{i \in T} \frac{1}{N} k(\state{i}, \state{n})}} \\
   &= \V(n)  + c \sqrt{\frac{\SRprior(\state{n})}{\sum_{i \in T} N \frac{1}{N} k(\state{i}, \state{n})}} \\
   &= \V(n)  + c \sqrt{\frac{\SRprior(\state{n})}{\sum_{i \in T} k(\state{i}, \state{n})}} \\
\end{align*}

Observing that $\SRprior$ is a uniform distribution, we see that $\SRprior(s)$ is constant for all $s$. We can absord this constant into $c$, which gives us the expression

\begin{align*}
   \frac{\partial}{\partial \srhat(n)}  \mathcal{L}(\srhat) &\approx \V(n)  + c \sqrt{\frac{1}{\sum_{i \in T} k(\state{i}, \state{n})}} \\
\end{align*}

Observe that this is the same exploration reward bonus used in Never Give Up \cite{Badia2020NGU}. 

Applying Lemma~\ref{lem:lemma4}, we can see this expressed in a more traditional reward based form. 

\begin{align*}
    \action(n) 
    &= \argmax_a \frac{\partial f}{\partial \hat{\pi}(a \mid n_t)} \\
    &= \argmax_a P(n) E_{n' \sim \srhat(\cdot \mid n, a)}\left[\frac{\partial}{\partial \srhat(n)}\right] \\
    &= \argmax_a E_{n' \sim \srhat(\cdot \mid n, a)} \left[\V(n)  + 1\sqrt{\frac{1}{\sum_{i \in T} k(\state{i}, \state{n})}} \right] \\
    &= \argmax_a E_{n' \sim \srhat(\cdot \mid n, a)}[\V(n)] 
    + E_{n' \sim \srhat(\cdot \mid n, a)}\left[c\sqrt{\frac{1}{\sum_{i \in T} k(\state{i}, \state{n})}} \right] \\
    &= \argmax_a Q^{\hat{\pi}}(s, a)   + c E_{n' \sim \srhat(\cdot \mid n, a)}\left[\sqrt{\frac{1}{\sum_{i \in T} k(\state{i}, \state{n})}}\right]
\end{align*}

Observe that $Q^{\hat{\pi}}(s, a)$ is the empirical average of future rewards -- this is exactly the $Q$ value calculated by traditional MCTS. $E_{n' \sim \srhat(\cdot \mid n, a)}\left[\sqrt{\frac{1}{\sum_{i \in T} k(\state{i}, \state{n})}}\right]$ is the empirical average of future values of $\sqrt{\frac{1}{\sum_{i \in T} k(\state{i}, \state{n})}}$. In other words, it is an additional term calculated the same way the value is calculated, but treating $\sqrt{\frac{1}{\sum_{i \in T} k(\state{i}, \state{n})}}$ as a supplemental reward. This is equivalent to adding $\sqrt{\frac{1}{\sum_{i \in T} k(\state{i}, \state{n})}}$ to the reward function, which is what count-based exploration rewards do. Therefore, this approximation is equivalent to a count-based exploration reward. 

\end{proof}

\subsubsection{Exploration efficiency} \label{Efficiency Proof}

Sampling-based motion algorithms frequently come with guarantees of exploration efficiency in addition to optimality. Reinforcement learning algorithms, on the other hand, rarely enjoy these kinds of guarantees outside of simple cases such as bandits. This is especially true in continuous domains. Recent work has established regret bounds for both MCTS and continuous-space generalizations (Note: early logarithmic regret bounds for MCTS are now thought to be incorrect, as they do not account for the value distribution being non-stationary). However, even these methods do not show that they efficiently explore the state space -- instead they show a bound on the regret as a function of $\gamma$, with the regret growing asymptotically as $\gamma \to 1$. We are able to provide polynomial bounds on the rate at which Volume-MCTS explores the state space, which do not depend on $\gamma$. This is a significant advantage for long-horizon problems where $\gamma$ may be very close to 1.

As a note, we derive these bounds for the idealized version of Volume-MCTS which uses 1-nearest neighbor as its density estimator, which we will call Voronoi-Volume-MCTS. This version of the algorithm is slightly different than the form presented in the main body of the paper, which uses a KD-tree as its density estimator. We choose to analyze this version instead, because the probability of expanding a node is proportional to the volume of its Voronoi region, rather than the volume of the KD-region. While KD-region volumes are much easier to calculate in practice, Voronoi regions are more mathematically tractable. This is primarily a artifact of the analysis rather than a meaningful feature of the math. Research in SBMP algorithms nearly always uses Voronoi regions for analysis, and approximates these regions using KD-trees in implementation. The difference is rarely relevant in practice.

We will begin by defining the following term
\begin{define}
    \textbf{$\delta$-controllable}: 
    Let $M$ be an MDP with action space $A$, bounded state space $S$, and deterministic transition function $\mathcal{T}(s_i, a_i)$. 
    Let $d_A$ be the dimensionality of $A$. 
    Let $\tau$ be a trajectory in $M$. 
    Let $s_i$ be the $i$-th state in the trajectory $\tau$. 
    Let $\B_\delta(s_i)$ be a ball of radius $\delta$ about $s_i$. 

    Then $\tau$ is $\delta$-controllable iff there exists a constant $\sigma > 0$ such that for each state $s_i$ in $\tau$, there exists a region in action space $A_i$ with measure at least $\sigma \delta^{d_A}$ such that if a state $s_i' \in \B_\delta(s_i)$ and $a_i' \in A_i$, then $\mathcal{T}(s_i', a_i') \in \B_\delta(s_{i+1})$.    
\end{define}

Intuitively, if we have a point close to a trajectory $\tau$, then we have a lower bounded chance of sampling an action that stays close to $\tau$ at the next state. This condition is strictly weaker than the assumptions used for asymptotically optimal motion planners. Stable Sparse RRT takes a set of assumptions that together necessitate that every point in the $\delta$-ball of $s_i$ is reachable from every point in the $\delta$-ball of $s_{i-1}$. By contrast, we only assume that the is a lower-bounded chance of reaching somewhere in the $\delta$-ball of $s_i$, without specifying where that may be or what shape it may have. 

Our strategy for this proof is as follows. First, we lower bound the probability of selecting a point near a state $s_t$ in the trajectory. Then, we lower bound the probability of reaching a state near $s_{t+1}$, given that a state near $s_t$  was reached. Finally, we sum over these bounds and use them to establish a bound of reaching an arbitrary region in a given amount of time.

Recall that the probability of expanding a node Voronoi-Volume-MCTS is given by $\sr(n) =  \frac{\lambda}{\alpha - \V(n)} \Vol(n)$, where 
\begin{itemize}
    \item $\V(n)$ is the value of the node $n$
    \item $\Vol(n)$ is the measure of $n$'s Voronoi region
    \item $\lambda = \frac{c}{\sqrt{N}}$
    \item $c$ is a constant
    \item $N$ is the current iteration number, and 
    \item $\alpha$ is whatever constant normalizes $\srstar(n)$, so it sums to 1. 
\end{itemize}

We aim to lower-bound the sum of all $\sr(n)$ near $s_t$, assuming that at least one node is near $s_t$. We begin by bounding the sum of their Voronoi regions.


\begin{lem} \label{Voronoi Volume Lemma}
Let $s \in S$ be such that $\B_\delta(s) \subset S$. Suppose that there exists a tree node $n$ with $\state{n} \in \B_{\frac{2\delta}{5}}(s)$. 
Let $s' \in S$ be an arbitrary state in $S$. 
Let $s_{near}$ denote the nearest neighbor of $s'$ among all tree nodes. 

Suppose the state space $S$ is bounded. Further suppose, without loss of generality, that the volume of the full state space $S$ is 1. 

Then the union of the Voronoi regions of all nodes in $\B_\delta(s)$ has a volume of at least $|\B_{\frac{\delta}{5}}|$
\end{lem}

\begin{proof}
    Our proof closely follows the proof given by \citet{Kleinbort_Solovey_Littlefield_Bekris_Halperin_2019} for their Lemma 4. 

    Case 1: Suppose all nodes in the tree are in $\B_\delta(s)$, then the union of their Voronoi regions is $S$. Then it is trivial that the union of their Voronoi regions has measure $> |\B_{\frac{\delta}{5}}|$. 

    Case 2: Suppose there is a tree node with $z \notin \B_\delta(s)$. 
    We show that if $s' \in B_{\frac{\delta}{5}}(s)$ then $s_{near} \in \B_\delta(s)$. 
    Observe that $\state{n} \in \B_{\frac{2\delta}{5}}(s)$, so by the triangle inequality, $||s' - \state{n}|| \leq \frac{3 \delta}{5}$. 
    Since $z \notin \B_\delta(s)$, $||s_{near} - z|| \geq \frac{4 \delta}{5}$. Therefore, $||s' - \state{n}|| < ||s' - z||$. Hence, $z$ is not the nearest neighbor of $s'$.
    Since $z$ was chosen arbitrarily from nodes outside of $\B_\delta(s)$, this holds for all nodes outside of $\B_\delta(s)$. It follows that $s'$ is in the Voronoi region of some node within $\B_\delta(s)$. Since $s'$ was again chosen arbitrarily from points in $B_{\frac{\delta}{5}}(s)$, the union of the Voronoi regions of nodes node in $\B_\delta(s)$ includes all points in $B_{\frac{\delta}{5}}(s)$. Observe that the volume of this region is $|\B_{\frac{\delta}{5}}|$. 
    
    Therefore,  the union of the Voronoi regions of all nodes in $\B_\delta(s)$ has a volume of at least $|\B_{\frac{\delta}{5}}|$. 
\end{proof}

Next, we must show bounds on $\alpha$. Together with Lemma \ref{Voronoi Volume Lemma}, this will give us a lower bound on the probability of sampling a node in a given region once that region has been reached. 

\begin{lem}
    1: $$\alpha \geq max_n(\V(n) + \lambda \Vol(n))$$ \\
    and 
    2: $$\alpha \leq max_n(\V(n)) + \lambda$$
\end{lem}

\begin{proof}
    1:  $\alpha = \V(n) + \lambda \frac{\Vol(n)}{\sr(n)}$. 
    Since $\sr(n) \leq 1$, $\alpha = \V(n) + \lambda \Vol(n)$ for all $n$. Hence, $\alpha \geq max_n(\V(n) + \lambda \Vol(n))$. 

    2: $\sum_n \sr(n) = 1$. Therefore, $1 = \sum_n \frac{\lambda \Vol(n)}{\alpha - \V(n)} \geq \sum_n \frac{\lambda \Vol(n)}{\alpha - max_{n'}(\V(n'))}$. 
    Hence, $\alpha - max_{n'}(\V(n')) \geq \sum_n \lambda \Vol(n) = \lambda$. Therefore, $\alpha \geq max_{n}(\V(n)) + \lambda$.     
\end{proof}

From this result, it is easy to derive the following bound on $\sr$: 
\begin{lem} \label{sr bound}
    Suppose, without loss of generality, that $0 \leq R \leq 1$. Then, 

    $\sr(n) \geq \frac{c(1-\gamma)}{\sqrt{N} + c(1-\gamma)}  \Vol(n)$
\end{lem}

\begin{proof}
    $0 \leq R \leq 1$, so $0 \leq \V(n) \leq \frac{1}{1-\gamma}$ for all $n$. Hence,
    
    \begin{align*}
        \sr(n) &= \frac{\lambda \Vol(n)}{\alpha - \V(n)} \\
            &\geq \frac{\lambda \Vol(n)}{max_n(\V(n)) + \lambda - \V(n)} \\
            &\geq \frac{\lambda \Vol(n)}{\frac{1}{1-\gamma}+ \lambda - \V(n)} \\
            &\geq \frac{\lambda \Vol(n)}{\frac{1}{1-\gamma}+ \lambda - 0} \\
            &\geq \frac{\lambda \Vol(n)}{\frac{1}{1-\gamma}+ \lambda} \\
            &\geq \frac{\lambda \Vol(n)}{\frac{1}{1-\gamma}+ \lambda} \\
            &\geq \Vol(n) \frac{c}{\sqrt{N}(\frac{1}{1-\gamma}+ \frac{c}{\sqrt{N}} )} \\
            &\geq \Vol(n) \frac{c}{\frac{\sqrt{N}}{1-\gamma}+ c} \\
            &\geq \Vol(n) \frac{c(1-\gamma)}{\sqrt{N}+ c(1-\gamma)}
    \end{align*}    
\end{proof}

Now that we have established bounds on $\sr(n)$, we can establish lower bounds on the probability of sampling a node near $s$ once $\B_\delta(s)$ has been reached. 

\begin{cor} \label{bd sample bound}
    If the ball $\B_\delta(s)$ has been reached, then the probability of expanding a node in $\B_\delta(s)$ at time $N$ is at least $|\B_{\frac{\delta}{5}}| \frac{c(1-\gamma)}{\sqrt{N} + c(1-\gamma)}$. 
\end{cor}

\begin{cor}
    For $N \geq c^2 (1-\gamma)^2$, $\sr(n) \geq \frac{c(1-\gamma)}{2\sqrt{N}} \Vol(n) $ for all nodes $n$.\\
    For $N \leq c^2 (1-\gamma)^2$, $\sr(n) \geq \frac{1}{2} \Vol(n) $ for all nodes $n$.\\
\end{cor}

\begin{proof}
    Suppose $N \geq c^2 (1-\gamma)^2$.  Recall that $\sr(n) \geq \Vol(n) \frac{1}{\frac{\sqrt{N}}{c(1-\gamma)}+ 1}$. Then $1 \leq \frac{\sqrt{N}}{c(1-\gamma)}$. Hence, $\Vol(n) \frac{1}{\frac{\sqrt{N}}{c(1-\gamma)}+ 1} \geq \Vol(n) \frac{1}{\frac{\sqrt{N}}{c(1-\gamma)}+ \frac{\sqrt{N}}{c(1-\gamma)}} = \Vol(n) \frac{1}{2\frac{\sqrt{N}}{c(1-\gamma)}} = \Vol(n) \frac{{c(1-\gamma)}}{2\sqrt{N}}$

    Suppose $N \leq c^2 (1-\gamma)^2$. Recall that $\sr(n) \geq \Vol(n) \frac{1}{\frac{\sqrt{N}}{c(1-\gamma)}+ 1}$. 
    Then $1 \geq \frac{\sqrt{N}}{c(1-\gamma)}$. 
    Hence, $\Vol(n) \frac{1}{\frac{\sqrt{N}}{c(1-\gamma)}+ 1} \geq \Vol(n) \frac{1}{1 + 1} = \frac{1}{2} \Vol(n)$
\end{proof}

\begin{cor}
    $\sr(n) \geq \frac{1}{2} \min(1, \frac{c(1-\gamma)}{\sqrt{N}}) \Vol(n) $ for all nodes $n$.\\
\end{cor}


Now, we can provide a lower bound on the probability that a state in a trajectory will be reached after a given expansion. 

Our goal is to take a $\delta$-controllable trajectory and cover each state in a $\delta$-ball. Then we will find a lower bound on the probability of reaching each of these balls in sequence. This will provide us with a high-probability bound on the time it will take to reach the last ball in the sequence.

\begin{lem}
    Let $M$ be an MDP with action space $A$, bounded state space $S$, and deterministic transition function $\mathcal{T}(s_i, a_i)$. 
    Let $d_A$ be the dimensionality of $A$. 
    Let $\tau$ be a $\delta$-controllable trajectory. 
    Let $\Bd{i}$ be the $\delta$-ball around $\tau_i$, the $i$-th state in $\tau$. 
    Then $\Bd{i+1}$ will be reached by time $N_{i + 1}$ with probability $ \exp \left(- |\B_{\frac{\delta}{5}}| c(1-\gamma) \sigma \delta^{d_A}  ( 2\sqrt{N_{i + 1}} - 2\sqrt{N_{i}} )\right) $
\end{lem}

\begin{proof}
    $\Bd{i+1}$ will be reached if we expand a node in $\Bd{i}$ and then sample an action that takes us to $\Bd{i+1}$. At timestep $N \geq c^2(1-\gamma)^2 $, a node in $\B_\delta(s)$ has a probability of at least $|\B_{\frac{\delta}{5}}| \frac{1}{2} \min \left(1, \frac{c(1-\gamma)}{\sqrt{N}} \right) $ of being expanded. 
    Since $\tau$ is a $\delta$-controllable trajectory, we have a probability of at least $\sigma \delta^{d_A}$ of sampling an action that takes the agent to a point in $\Bd{i+1}$ if a node in $\Bd{i}$ is sampled. 
    Hence, we have a probability of at least $$|\B_{\frac{\delta}{5}}| \sigma \delta^{d_A} \frac{1}{2} \min \left(1, \frac{c(1-\gamma)}{\sqrt{N}} \right) $$  of reaching the next ball in the sequence at each step. Observe that these samples are drawn independently, so the chance of failing many times in a row is the product of the probability of failure at each time. 
    Then, the probability of failing to reach $\Bd{i+1}$ by time $N_{i + 1}$ is less than or equal to $$\prod_{t=N_i}^{N_{i + 1}} 1 - |\B_{\frac{\delta}{5}}| \sigma \delta^{d_A} \frac{1}{2} \min \left(1, \frac{c(1-\gamma)}{\sqrt{t}}\right) $$

    We then have 
    \begin{align*}
        P(\Bd{i+1} \text{ not reached} \mid \Bd{i}  \text{ reached} ) 
        &\leq \prod_{t=N_i}^{N_{i + 1}} 1 - |\B_{\frac{\delta}{5}}| \sigma \delta^{d_A} \frac{1}{2} \min \left(1, \frac{c(1-\gamma)}{\sqrt{t}} \right)  \\
        &\leq \exp \left(\sum_{t=N_i}^{N_{i + 1}} \ln \left(1 - |\B_{\frac{\delta}{5}}| \sigma \delta^{d_A} \frac{1}{2} \min \left(1, \frac{c(1-\gamma)}{\sqrt{t}} \right)  \right) \right) \\
        &\leq \exp \left(- \sum_{t=N_i}^{N_{i + 1}} |\B_{\frac{\delta}{5}}| \sigma \delta^{d_A} \frac{1}{2} \min \left(1, \frac{c(1-\gamma)}{\sqrt{t}} \right) \right)  \\
        &\leq \exp \left(- |\B_{\frac{\delta}{5}}|\sigma \delta^{d_A} \sum_{t=N_i}^{N_{i + 1}}  \frac{1}{2} \min \left(1, \frac{c(1-\gamma)}{\sqrt{t}} \right)  \right)  \\
    \end{align*}

    This summation is difficult to analyze, so we will instead bound it with an integral that is more tractable. 

    Observe that $\frac{c(1-\gamma)}{2\sqrt{t}}$ is non-increasing in $t$. Therefore we can apply the bound $$\sum_{t=N_i}^{N_{i + 1}}  \min \left(1, \frac{c(1-\gamma)}{\sqrt{t}} \right) \geq \int_{N_i}^{N_{i + 1}} \min \left(1, \frac{c(1-\gamma)}{\sqrt{t}} \right) dt$$

    We now attempt to simplify the right hand side. The min operation here produces 3 cases. 
    \begin{enumerate}
        \item $N_i < N_{i+1} < c^2(1-\gamma)^2$
        \item $N_i < c^2(1-\gamma)^2 < N_{i+1}$
        \item $c^2(1-\gamma)^2 < N_i < N_{i+1}$
    \end{enumerate}

    In the first case, $\int_{N_i}^{N_{i + 1}} \min \left(1, \frac{c(1-\gamma)}{\sqrt{t}} \right) dt = \int_{N_i}^{N_{i + 1}} 1 dt = N_{i + 1} - N_i$. \\
    In the second case, we must first split the integral into portions covering $t < c^2(1-\gamma)^2$ and $t \geq c^2(1-\gamma)^2$
    \begin{align*}
        \int_{N_i}^{N_{i + 1}} \min \left(1, \frac{c(1-\gamma)}{\sqrt{t}} \right) dt 
        &= \int_{N_i}^{c^2(1-\gamma)^2} 1 dt + \int_{c^2(1-\gamma)^2}^{N_{i+1}} \frac{c(1-\gamma)}{\sqrt{t}}) dt \\
        &= c^2(1-\gamma)^2 - N_i +  c(1-\gamma) (\sqrt{N_{i + 1}} - \sqrt{c^2(1-\gamma)^2}) \\
        &= c^2(1-\gamma)^2 - N_i +  c(1-\gamma) \sqrt{N_{i + 1}} - c^2(1-\gamma)^2 \\
        &= c(1-\gamma) \sqrt{N_{i + 1}} - N_i \\
    \end{align*}
    In the third case, $\int_{N_i}^{N_{i + 1}} \min \left(1, \frac{c(1-\gamma)}{\sqrt{t}} \right) dt = \int_{N_i}^{N_{i + 1}} \frac{c(1-\gamma)}{\sqrt{t}} dt  = c(1-\gamma) (\sqrt{N_{i + 1}} - \sqrt{N_{i}})$. \\ 
    We can simplify this solution to 
    \begin{align*}
        \min \left(N_{i + 1}, c(1-\gamma) \sqrt{N_{i + 1}} \right) - \min(N_{i}, c(1-\gamma) \sqrt{N_{i}})
        &= \min \left(t, c(1-\gamma) \sqrt{t} \right) \mid_{N_{i}}^{N_{i + 1}}
    \end{align*}
    Hence,  
    \begin{align*}
        \sum_{t=N_i}^{N_{i + 1}}  \min \left(1, \frac{c(1-\gamma)}{\sqrt{t}} \right) &\geq \int_{N_i}^{N_{i + 1}} \min \left(1, \frac{c(1-\gamma)}{\sqrt{t}} \right) dt \\
            &= \min \left(t, c(1-\gamma) \sqrt{t} \right) \mid_{N_{i}}^{N_{i + 1}}
    \end{align*}

    We now see that 
     \begin{align*}
        P(\Bd{i+1} \text{ not reached} \mid \Bd{i}  \text{ reached}) 
        &\leq \exp \left(- |\B_{\frac{\delta}{5}}|\sigma \delta^{d_A} \sum_{t=N_i}^{N_{i + 1}}  \frac{1}{2} \min \left(1, \frac{c(1-\gamma)}{\sqrt{t}} \right)  \right) \\
        &\leq \exp \left(- \frac{1}{2}  |\B_{\frac{\delta}{5}}| \sigma \delta^{d_A} \min \left(t, c(1-\gamma) \sqrt{t} \right) \mid_{N_{i}}^{N_{i + 1}} \right)
    \end{align*}

    \end{proof}

    




    Now that we have a bound on the probability of reaching the next node within a fixed time frame, we will use this to find a bound on the probability of traversing a sequence of points. In other words, we need an upper bound on $P(\Bd{i+1} \text{ not reached by time $N$})$. To achieve this, we will use the law of total probability, defining 
    \begin{align*}
        &P(\Bd{i+1} \text{ reached by time $N$}) \\
        &= \sum^{N_{i+1}}_{t=0} P(\Bd{i+1} \text{ reached by time $N$} |  \Bd{i} \text{ reached at time $t$}) P(\Bd{i} \text{ reached at time $t$})    
    \end{align*}
    However, we do not know the exact probability $P(\Bd{i} \text{ reached at time $t$})$. 
    Instead, we will show that if we have a lower bound $LB_i(t)$ on the probability that $\Bd{i}$ has been reached by time $t$ such that $\lim_{t\to\inf0y} LB_i(t) = 1$, then we can define $M_i$ to be the PDF of this upper bound: $M_i(t) = LB_i(t) - LB_{i}(t-1)$. We show that $\sum^{N_{i+1}}_{t=0} P(\Bd{i+1} \text{ reached by time $t$}) M_i(t)$ is then a lower bound on $P(\Bd{i+1} \text{ reached by time $N$})$. This is true as long as $P(\Bd{i+1} \text{ reached by time $t$})$ is monotonically decreasing in $t$.

    \begin{lem} \label{Monotonic bounding lemma}
        Let $A(x)$ with $x \in [0, \infty)$ be a function that integrates to $1$. 
        Let $B(x)$ with $x \in [0, \infty)$ be a function that integrates to $1$. 
        Suppose that for any $k$, $\int_{0}^k B(x) dx \leq \int_{0}^k A(x) dx$. 
        Then if a function $F(x)$ is non-negative and non-decreasing, then
        $\int_{0}^\infty B(x) F(x) dx \geq \int_{0}^\infty A(x) F(x) dx$. Similarly, if $F$ is non-negative and non-increasing, then  $\int_{0}^k B(x) F(x) dx \leq \int_{0}^k A(x) F(x) dx$.
    \end{lem}
    
    \begin{proof}
        Observe that $\int_{0}^\infty B(x) F(x) dx \geq \int_{0}^\infty A(x) F(x) dx$ iff $\int_{0}^\infty (A(x) - B(x)) F(x) dx \leq 0$.  
        
        First, recall that for any $k$, $\int_{0}^k B(x) dx \leq \int_{0}^k A(x) dx$. 
        Then, 
        \begin{align*}
            \int_{0}^k B(x) dx &\leq \int_{0}^k A(x) dx \\
            1-\int_{0}^k B(x) dx &\geq 1-\int_{0}^k A(x) dx \\
            \int_{k}^\infty B(x) dx &\geq \int_{k}^\infty A(x) dx \\
            0 &\geq \int_{k}^\infty (A(x) - B(x)) dx
        \end{align*}

                Note that because $F(x)$ is non-negative and non-decreasing, $F'(x) \geq 0$ for all $x$. 
        
        We now introduce an additional integration variable. This will allow us to rearrange to the integral and make the proof easier. Observe that $F(x) - F(0) = \int_0^x F'(y) dy$
        Then, 
        \begin{align*}
            \int_{0}^\infty (A(x) - B(x)) F(x) dx 
            &= \int_{0}^\infty (A(x) - B(x)) (\int_0^x F'(y) dy - F(0)) dx\\
            &= \int_{0}^\infty (A(x) - B(x)) \int_0^x F'(y) dy dx - F(0) \int_{0}^\infty (A(x) - B(x)) dx\\
            &= \int_{0}^\infty (A(x) - B(x)) \int_0^x F'(y) dy dx - F(0) \int_{0}^\infty A(x) dx + F(0) \int_{0}^\infty B(x) dx\\
            &= \int_{0}^\infty (A(x) - B(x)) \int_0^x F'(y) dy dx - F(0) (1) + F(0) (1)\\
            &= \int_{0}^\infty (A(x) - B(x)) \int_0^x F'(y) dy dx\\
            &= \int_{0}^\infty  \int_0^x (A(x) - B(x)) F'(y) dy dx
        \end{align*}

        By Fubini's theorem, we can then rearrange the integrals as follows
         \begin{align*}
            \int_{0}^\infty (A(x) - B(x)) F(x) dx
            &= \int_{0}^\infty  \int_0^x (A(x) - B(x)) F'(y) dy dx \\
            &= \int_{0}^\infty  \int_y^\infty (A(x) - B(x)) F'(y) dx dy \\
            &= \int_{0}^\infty  F'(y) \int_y^\infty (A(x) - B(x)) dx dy 
        \end{align*}

        Observe that $\int_y^\infty (A(x) - B(x)) dx \leq 0$, and $F'(y) \geq 0$ for all $y$. Therefore, $F'(y) \int_y^\infty (A(x) - B(x)) dx  \leq 0$  for all $y$. It then follows that $F'(y) \int_y^\infty (A(x) - B(x)) dx \leq 0$. By our earlier observation, this implies that $\int_{0}^\infty B(x) F(x) dx \geq \int_{0}^\infty A(x) F(x) dx$.         
    \end{proof}


    \begin{thm}
        Let $\tau$ be a $\delta$-controllable trajectory, with states $s_0 ... s_L$. 
        Let $d_A$ be the dimension of the action space. 
        Let $\Bd{i}$ be the $\delta$-ball around $\tau_i$, the $i$-th state in $\tau$. 

        Then the probability that $\Bd{i}$ will be reached after $N$ expansions is lower-bounded by $1-\frac{\Gamma(i, \frac{1}{2} |\B_{\frac{\delta}{5}}| \sigma \delta^{d_A} c(1-\gamma) (\sqrt{N_1} - \sqrt{t_0})))}{\Gamma(i)}$
    \end{thm}
    
    \begin{proof}
        Note that the first state in the trajectory is the starting state $s_0$, which is reached by the first time step.  
        Assuming that $\Bd{i-1}$ was reached at time $N_{i-1}$, $\Bd{i}$ will be reached by time $N_i$ with probability $1 - \exp \left(- |\B_{\frac{\delta}{5}}| \sigma \delta^{d_A} \min(t, c(1-\gamma) \sqrt{t}) \mid_{N_{i}}^{N_{i + 1}} \right)$. We can integrate over the time that $\Bd{i-1}$ was reached to find a tight bound $LB_i(N_i)$ on the probability of reaching $\Bd{i}$ by time $N_i$. 

        

        
        We aim to find a closed-form expression $LB_i(t)$ for all $i, t$. 
        We do this by proof by induction. We show that there exists $LB_i(N_i)$ such that 
        \begin{enumerate}
            \item $P(\Bd{i} \text{ reached by } N_{i}) \geq LB_i(N_i)$
            \item For $N_i < t_0$, $LB_{i}(N_{i}) = 0$
            \item $\lim_{N_{i} \to \infty} LB_{i}(N_{i}) = 1$ 
        \end{enumerate}

        Let $t_0 = c^2 (1-\gamma)^2$. 

         Since the first region is reached when the problem begins, it is clear that $P(\Bd{1} \text{ reached by } t) = 1$ for all $t \geq 1$. 
        However, we will find that it is more convenient to use the lower bound for the probability $LB_0(t) = 0$ when $t < t_0$ and $LB_0(t) = 1$ when $t \geq t_0$.
        This trivially gives a lower bound for $P(\Bd{1} \text{ reached by } N_1) 
        \geq LB_1(N_1)$, where $LB_1(N_1) = 0$ for $N_1  < t_0$ and
        $LB_1(N_1) = 1 - \exp \left(-\frac{1}{2}  |\B_{\frac{\delta}{5}}| \sigma \delta^{d_A} \min(t, c(1-\gamma) \sqrt{t}) \mid_{t_0}^{N_{1}} \right) = 1 - \exp \left(- |\B_{\frac{\delta}{5}}| \sigma \delta^{d_A} c(1-\gamma) (\sqrt{N_1} - c(1-\gamma)) \right)$ for $N_1 \geq t_0$.

            Let $C = \frac{1}{2} |\B_{\frac{\delta}{5}}| \sigma \delta^{d_A} c(1-\gamma)$. 
            Let $LB_i(N_i) = 0$ for $N_i <  t_0$. Then $P(\Bd{i} \text{ reached by } N_{i}) \geq LB_i(N_i)$ for $N_i <  t_0$.. 
            For the rest of the derivation, we work under the assumption that $N_i \geq  t_0$.

            Then 
            \begin{align*}
                P(\Bd{i} \text{ reached by } N_i) &= \sum_{N_{i-1}= t_0}^{N_i} P(\Bd{i} \text{ reached by } N_i \mid \Bd{i-1} \text{ reached at } N_{i-1}) P(\Bd{i-1} \text{ reached at } N_{i-1}) \\
                &\geq \sum_{N_{i-1}= t_0}^{N_i} \left(1 - \exp \left(- C (\sqrt{N_{i}} - \sqrt{N_{i-1}}) \right) \right)  P(\Bd{i-1} \text{ reached at } N_{i-1}) \\
                &\geq \sum_{N_{i-1}= t_0}^{N_i} (1 - \exp \left(- C (\sqrt{N_{i}} - \sqrt{N_{i-1}}) \right))  P(\Bd{i-1} \text{ reached at } N_{i-1}) \\
                &\geq \sum_{N_{i-1}= t_0}^{\infty} \max \left( 0 , 1- \exp \left(- C (\sqrt{N_{i}} - \sqrt{N_{i-1}}) \right) \right) P(\Bd{i-1} \text{ reached at } N_{i-1}) \\
                & \geq \int_{t_0}^{\infty} \max \left( 0 , 1- \exp \left(- C (\sqrt{N_{i}} - \sqrt{N_{i-1}}) \right) \right) P(\Bd{i-1} \text{ reached at } \floor{N_{i-1}})  d N_{i-1}\\ 
            \end{align*}

            Observe that $\max \left( 0 , 1- \exp \left(- C (\sqrt{N_{i}} - \sqrt{N_{i-1}}) \right) \right)$ is non-negative and non-decreasing in $N_{i-1}$. Additionally,  $LB_{i-1} \leq P(\Bd{i-1} \text{ reached by } N_{i-1})$. This means we can apply Lemma \ref{Monotonic bounding lemma}, using  $\frac{d LB_{i-1}}{d N_{i-1}}(N_{i-1})$ to bound $P(\Bd{i-1} \text{ reached at } \floor{N_{i-1}})$. 
            \begin{align*}
                P(\Bd{i} \text{ reached by } N_i) &\geq \int_{t_0}^{\infty} \max \left( 0 , 1- \exp \left(- C (\sqrt{N_{i}} - \sqrt{N_{i-1}}) \right) \right) P(\Bd{i-1} \text{ reached at } \floor{N_{i-1}})  d N_{i-1}\\ 
                 &\geq \int_{t_0}^{\infty} \max \left( 0 , 1- \exp \left(- C (\sqrt{N_{i}} - \sqrt{N_{i-1}}) \right) \right)   \frac{d LB_{i-1}}{d N_{i-1}}(N_{i-1}) d N_{i-1} \\
                 &\geq \int_{t_0}^{N_i} \max \left( 0 , 1- \exp \left(- C (\sqrt{N_{i}} - \sqrt{N_{i-1}}) \right) \right)   \frac{d LB_{i-1}}{d N_{i-1}}(N_{i-1}) d N_{i-1} \\
                 &\geq \int_{t_0}^{N_i} \left(1- \exp \left(- C (\sqrt{N_{i}} - \sqrt{N_{i-1}}) \right) \right)   \frac{d LB_{i-1}}{d N_{i-1}}(N_{i-1}) d N_{i-1} \\
                 &\geq \int_{t_0}^{N_i} \left(1-  \exp \left(- C \sqrt{N_{i}}\right) \exp \left(C\sqrt{N_{i-1}}) \right) \right)  \frac{d LB_{i-1}}{d N_{i-1}}(N_{i-1}) d N_{i-1} \\
            \end{align*}

            We have then shown that $LB_i(N_i) = 1 - \exp \left(- C \sqrt{N_{i}}\right) \int_{t_0}^{N_i} \exp \left(C\sqrt{N_{i-1}}) \right)   \frac{d LB_{i-1}}{d N_{i-1}}(N_{i-1}) d N_{i-1}$ is a lower bound on  $P(\Bd{i} \text{ reached by } N_i)$ if $LB_{i-1}(N_{i-1})$ is a lower bound on $P(\Bd{i-1} \text{ reached by } N_{i-1})$. However, we now see that $\frac{d LB_{i}}{dN_{i}}(N_{i})$ is the more immediately useful term, because it is what appears in our bound. If we find $\frac{d LB_{i}}{dN_{i}}(N_{i})$ for all $N_i$, we can use this recurrence relation to calculate a closed-form bound. With this in mind, we now solve for $\frac{d LB_{i}}{dN_{i}}$. 
            \begin{align*}
                \frac{d LB_{i}}{dN_{i}} &= \frac{d }{d N_i} \int_{t_0}^{N_i} \left(1-  \exp \left(- C \sqrt{N_{i}}\right) \exp \left(C\sqrt{N_{i-1}}) \right) \right)  \frac{d LB_{i-1}}{d N_{i-1}}(N_{i-1}) d N_{i-1} \\
                &= \frac{d LB_{i-1}}{d N_{i-1}}(N_{i})  - \frac{d }{d N_i} \int_{t_0}^{N_i}  \exp \left(- C \sqrt{N_{i}}\right) \exp \left(C\sqrt{N_{i-1}}) \right) \frac{d LB_{i-1}}{d N_{i-1}}(N_{i-1}) d N_{i-1} \\  
                &= \frac{d }{d N_i} \int_{t_0}^{N_i} \frac{d LB_{i-1}}{d N_{i-1}}(N_{i-1}) d N_{i-1} - \frac{d }{d N_i} \exp \left(- C \sqrt{N_{i}}\right) \int_{t_0}^{N_i}   \exp \left(C\sqrt{N_{i-1}}) \right) \frac{d LB_{i-1}}{d N_{i-1}}(N_{i-1}) d N_{i-1} \\  
                &= \frac{d LB_{i-1}}{d N_{i-1}}(N_{i}) - \frac{C \exp \left(- C \sqrt{N_{i}}\right)}{2 \sqrt{N_{i}}} \int_{t_0}^{N_i}   \exp \left(C\sqrt{N_{i-1}} \right) \frac{d LB_{i-1}}{d N_{i-1}}(N_{i-1}) d N_{i-1} \\&-  \exp \left(- C \sqrt{N_{i}}\right)\frac{d }{d N_i} \int_{t_0}^{N_i}   \exp \left(C\sqrt{N_{i-1}}) \right) \frac{d LB_{i-1}}{d N_{i-1}}(N_{i-1}) d N_{i-1}\\  
                &= \frac{d LB_{i-1}}{d N_{i-1}}(N_{i}) - \frac{C \exp \left(- C \sqrt{N_{i}}\right)}{2 \sqrt{N_{i}}} \int_{t_0}^{N_i}   \exp \left(C\sqrt{N_{i-1}} \right) \frac{d LB_{i-1}}{d N_{i-1}}(N_{i-1}) d N_{i-1} \\&-  \exp \left(- C \sqrt{N_{i}}\right) \exp \left(C\sqrt{N_{i}}) \right) \frac{d LB_{i-1}}{d N_{i-1}}(N_{i})\\  
                &= \frac{d LB_{i-1}}{d N_{i-1}}(N_{i}) - \frac{C \exp \left(- C \sqrt{N_{i}}\right)}{2 \sqrt{N_{i}}} \int_{t_0}^{N_i}   \exp \left(C\sqrt{N_{i-1}} \right) \frac{d LB_{i-1}}{d N_{i-1}}(N_{i-1}) d N_{i-1} -  \frac{d LB_{i-1}}{d N_{i-1}}(N_{i-1})\\  
                &=  - \frac{C \exp \left(- C \sqrt{N_{i}}\right)}{2 \sqrt{N_{i}}} \int_{t_0}^{N_i}   \exp \left(C\sqrt{N_{i-1}} \right) \frac{d LB_{i-1}}{d N_{i-1}}(N_{i-1}) d N_{i-1}
            \end{align*}


            

    We now show by induction that the general solution to this is $\frac{d LB_{i}}{d N_{i}} = 0$ for $N_{i} < t_0$ and $\frac{d LB_{i}}{d N_{i}} = \frac{C^i}{2 (i-1)!} \exp(-C \sqrt{N_{i}}) \frac{(\sqrt{N_{i}} - \sqrt{t_0})^{i-1}}{\sqrt{N_{i}}}$ for $N_{i} \geq t_0$, for all $i \geq 1$. 
    
    \textbf{Base Case:}
    Let $i=1$
    Recall that $LB_0(N_0) = 1$ for all $N_0 > t_0$, and $LB_1(N_1) = 1 - \exp \left(- C (\sqrt{N_1} - c(1-\gamma)) \right)$. 
    \begin{align*}
        \frac{d LB_1}{d N_1} &= \frac{C \exp \left(- C (\sqrt{N_1} - c(1-\gamma)) \right)}{2 \sqrt{N_1}} \\
        &= \frac{C}{2} N_1^{-\frac{1}{2}} \exp \left(- C (\sqrt{N_1} - c(1-\gamma)) \right) \\ 
        &= - \frac{C^1}{2 1!} \frac{(\sqrt{N_{1}} - \sqrt{t_0})^0}{\sqrt{N_{i}}} \exp \left(- C (\sqrt{N_1} - c(1-\gamma)) \right)\\
    \end{align*}

    Thus, $\frac{d LB_{i}}{d N_{i}} = \frac{C^i}{2 (i-1)!} \exp(-C \sqrt{N_{i}}) (N_{i})^{\frac{i-2}{2}}$ for $i = 1$.

    \textbf{Inductive case:}
    By the inductive hypothesis, 
     \begin{align*}    
     &\frac{C \exp(-C \sqrt{N_{i+1}})}{2 \sqrt{N_{i+1}}} \int_{t_0}^{N_{i+1}} \frac{d LB_{i}}{d N_{i}} \exp(C \sqrt{N_i}) dN_i \\
    &= \frac{C \exp(-C \sqrt{N_{i+1}})}{2 \sqrt{N_{i+1}}} \int_{t_0}^{N_{i+1}} \frac{C^{i}}{2 (i-1)!} \exp(-C \sqrt{N_i}) \frac{(\sqrt{N_{i}} - \sqrt{t_0})^{i-1}}{\sqrt{N_{i}}} \exp(C \sqrt{N_i})) dN_i \\
    &= \frac{C \exp(-C \sqrt{N_{i+1}})}{2 \sqrt{N_{i+1}}} \int_{t_0}^{N_{i+1}} \frac{C^{i}}{2 (i-1)!} \frac{(\sqrt{N_{i}} - \sqrt{t_0})^{i-1}}{\sqrt{N_{i}}} dN_i \\
    &= \frac{C \exp(-C \sqrt{N_{i+1}})}{2 \sqrt{N_{i+1}}} (\frac{C^{i}}{2 (i-1)! (\frac{i}{2})} (\sqrt{N_{i}} - \sqrt{t_0})^{i}\mid_{t_0}^{N_{i+1}}) \\
    &= \frac{C \exp(-C \sqrt{N_{i+1}})}{2 \sqrt{N_{i+1}}} (\frac{C^{i}}{i!} (\sqrt{N_{i}} - \sqrt{t_0})^{i} \mid_{t_0}^{N_{i+1}}) \\
    &= \frac{C \exp(-C \sqrt{N_{i+1}})}{2 \sqrt{N_{i+1}}} \frac{C^{i}}{i!} ((\sqrt{N_{i}} - \sqrt{t_0})^{i}  - 0^{\frac{i}{2}}) \\
    \end{align*}
    Observe that $i \geq 1$, so $0^\frac{i}{2} = 0$
    \begin{align*}
    \frac{C \exp(-C \sqrt{N_{i+1}})}{2 \sqrt{N_{i+1}}} \int_{t_0}^{N_{i+1}} \frac{d LB_{i}}{d N_{i}} \exp(C \sqrt{N_i}) dN_i &= \frac{C \exp(-C \sqrt{N_{i+1}})}{2 \sqrt{N_{i+1}}} \frac{C^{i}}{i!} (\sqrt{N_{i}} - \sqrt{t_0})^{i} \\
    &= \frac{C^{i+1}}{2i!} \frac{\exp(-C \sqrt{N_{i+1}})}{\sqrt{N_{i+1}}} (\sqrt{N_{i}} - \sqrt{t_0})^{i} \\
    &= \frac{C^{i+1}}{2i!} \exp(-C \sqrt{N_{i+1}}) \frac{(\sqrt{N_{i+1}} - \sqrt{t_0})^{i}}{\sqrt{N_{i+1}}} \\
    &= \frac{C^{i+1}}{2((i+1)-1)!} \exp(-C \sqrt{N_{i+1}}) \frac{(\sqrt{N_{i+1}} - \sqrt{t_0})^{i}}{\sqrt{N_{i+1}}}
    \end{align*}
Thus, the solution holds for the inductive case. 
    Hence, $\frac{C^i}{2 (i-1)!} \exp(-C \sqrt{N_i}) \frac{(\sqrt{N_{i}} - \sqrt{t_0})^{i-1}}{\sqrt{N_{i}}}$ is the solution for all $i \geq 1$.
    
    It follows that 
    \begin{align*}
        LB_i(N_i) &= \int_{t_0}^{N_i} \frac{d LB_i}{d T} dT \\
        &= \int_{t_0}^{N_i} \frac{C^i}{2 (i-1)!} \exp(-C \sqrt{T}) \frac{(\sqrt{T} - \sqrt{t_0})^{i-1}}{\sqrt{T}} dT \\
        &= \frac{C^i}{2 (i-1)!} \int_{t_0}^{N_i} \exp(-C \sqrt{T}) \frac{(\sqrt{T} - \sqrt{t_0})^{i-1}}{\sqrt{T}} dT \\
    \end{align*}
    
    Here, we can make an interesting observation -- this integral is in fact an incomplete Gamma function. Simplifying, we find that 
    
    \begin{align*}
        LB_i(N_i) &= \frac{C^i}{2 (i-1)!} (-2 C^{-i}) \Gamma(i, C(\sqrt{T} - \sqrt{t_0})) \mid_0^{N_i} \\
        &= \frac{1}{(i-1)!} (-1) \Gamma(i, C(\sqrt{T} - \sqrt{t_0})) \mid_0^{N_i}  \\
        &= \frac{1}{(i-1)!} (\Gamma(i)-\Gamma(i, C(\sqrt{N_1} - \sqrt{t_0}))) \\
        &=  1-\frac{\Gamma(i, C(\sqrt{N_1} - \sqrt{t_0})))}{(i-1)!} \\
        &=  1-\frac{\Gamma(i, C(\sqrt{N_1} - \sqrt{t_0})))}{\Gamma(i)} \\
    \end{align*}

    Hence, $P(\Bd{i} \text{ reached by } N_{i}) \geq 1-\frac{\Gamma(i, C(\sqrt{N_1} - \sqrt{t_0})))}{\Gamma(i)}$
            
    \end{proof}
    Observe that the form given is a Gamma distribution over the variable $(\sqrt{N_1} - \sqrt{t_0})$ with shape $i$ and rate $C$. 
    



\section{Experimental Details} \label{Appendix: Experiment Details}
\subsection{Hardware} All experiments were performed on an Alienware-Aurora-R9 with an 8-core Intel i7-9700 CPU. Since tree operations were the performance bottleneck, we did not use a graphics card for training. 

\subsection{Hyperparameters}
\textbf{AlphaZero and Volume-MCTS}: \\
For all AlphaZero variants, we set $\lambda = \frac{1}{(1-\gamma)\sqrt{N}}$ (equivalent to setting the exploration coefficient $c$ to $\frac{1}{1-\gamma}$ for AlphaZero). 
We chose this value due to an insight from our efficient exploration proof. The bound on time needed to reach new states depends on $\sqrt{N} - c(1-\gamma)$. Setting $c = \frac{1}{1-\gamma}$ makes $c(1-\gamma) = 1$, which minimizes the bound on exploration time. 

For the loss coefficients, we used
$c_{V}=1$\\
$c_{KL}=10$\\
$c_A=1$. \\
	
All neural nets use MLPs with ReLU activations and 3 hidden layers of 256 each. Training uses the Adam optimizer with the following hyperparameters \\
Learning rate = 0.001 \\
$\beta_1, \beta_2$ = [0.9, 0.99] \\
Weight decay = 0 \\
$\epsilon$ = 1e-07 \\
amsgrad: False \\ 
These hyperparameters were standard for the implementation our code was based on, and we did not change them. 

We did not otherwise do any extensive hyperparameter search. As much as possible, we used the hyperparameter settings from the existing implementation we were comparing to. These included the following hyperparameter values

\textbf{SST}: \\
Selection radius = 0.3 \\
Witness radius = 0.16 \\

\textbf{POLY-HOOT}: \\
HOO depth limit = 10 \\
$\alpha$ = 2.5 \\ 
$\xi$ = 10 \\
$\eta$ = 0.5 \\

\textbf{HER}:\\
Replay k = 4\\
Polyak averaging = 0.95\\
Entropy regularization = 0.01\\
Batch Size =  256\\
Batches per episode = 40\\

\subsection{Setting seeds} For all experiments, we repeat these experiments with three random seeds. We report the average and two-standard deviation confidence interval. 

 \subsection{Data Collection}
For all of the Maze environments, we performed 3 training runs, and gathered 10 samples from each.  We report the mean and 95\% confidence interval for each method. 

We found that the Quadcopter environment was significantly higher-variance, so we used more evaluations. Each method was run 60 times. For all the planning methods, this time was spend purely on search instead of learning. We found that planning was much more efficient per environmental interaction than learning, at least on the scale we evaluated for. For each HER run, we initialized a new neural net, trained it for the stated number of environmental interactions, and then evaluated it once. 

\subsection{Algorithm details}
	Beyond the algorithm described in the paper, there are a few problem-specific adaptations we make to the algorithms we study in order to improve convergence on the navigation environments. Firstly, we assume that there exists a "stay still" action in the action space that allows the agent to stay in the same state. This is important for two reasons. First, Volume-MCTS and Open-Loop AlphaZero are both open-loop algorithms -- they plan out a sequence of actions, and then follow that sequence without replanning at future steps. If they run out of actions in that sequence before the episode ends, the agent takes the "stay still" action until the episode ends. Since the agent always has access to this action, we also lower bound the value estimate for every state as $\frac{1}{1-\gamma}R(s)$, as the agent can always achieve this reward by just repeatedly selecting the "stay still" action. 
		
\textbf{Data collection}: AlphaZero only uses the root node of the tree for training. This works for closed-loop algorithms, because they run a search at each step of the episode, so they will perform searches with root nodes in every explored location. However, for open-loop algorithms, the root node is always the state that the agent starts the episode in, which may be a much more limited distribution. Therefore, we must use the entire search tree as learning data if we wish to train on data from the whole space. We use data from every node that has at least 1 action to train. 
	
\textbf{Action selection}:  MCTS also typically selects the action that has been explored the most times to be executed by the agent. \textit{MCTS as Regularized Policy Optimization} instead chooses to calculate the optimal policy exactly and then samples from it to take actions. For both methods, this is preferable to selecting the action with the highest value, because it encourages exploration. However, as a open-loop algorithm, Volume-MCTS selects an action only after it has completed all its search for the entire episode. The whole search tree can be built and stored for training before actions are executed. It never benefits from selecting suboptimal or exploratory actions for execution, because selecting these actions never leads to different data than it would get by taking the optimal action. 
		
Instead, both Volume-MCTS and Open-Loop AlphaZero keep track of the maximum actual earned reward of each branch, and always select the branch with the highest maximum value at the end of the episode. 

\subsection{Implementation} 

Our implementation draws on several existing codebases: an implementation of AlphaZero-Continuous by \citet{Moerland2018}, the pyOptimalMotionPlanning package developed by Kris Hauser (https://github.com/krishauser/pyOptimalMotionPlanning). For HER, we draw on Tianhong Dai's implementation of HER (https://github.com/TianhongDai/hindsight-experience-replay) and the implementation from the authors of USHER (https://github.com/schrammlb2/USHER\_Implementation) \cite{Schramm_Deng_Granados_Boularias_2022}.  Our POLY-HOOT implementation uses the author's implementation (https://github.com/xizeroplus/POLY-HOOT) \cite{POLY-HOOT}.

\subsection{Environment Details}

\subsubsection{Maze}
In this environment, the agent must navigate a maze to reach a goal. Episodes are 50 steps long. The reward function is 1 in the goal region, and 0 at all other states.  If an agent reaches the goal before the end of the episode, the episode ends and the agent receives a reward of 1 for each remaining time step left in the episode. 

We tested two sets of dynamics on the maze environment. Geometric dynamics are simple; the state space and action spaces are both 2-dimensional, and $s_{t+1} = s_t + v_{max} a_t$, where $s_{t+1}$ is the next state, $s_t$ is the current state, $a_t$ is the action, and $v_{\max}$ is the maximum speed allowed by the environment. If this movement would cause the agent to collide with a wall, instead the agent does not move ($s_{t+1} = s_t$). 

Dubins car dynamics are slightly more complicated. The state space has three dimensions: two position coordinates and one rotation coordinate. The action space is two-dimensional. The agent selects a forward/backward speed and a turning angle, which is bounded to give the agent a minimum turning radius. 
The dynamics are as follows: 
Let $x,y$ be the car's $x$ and $y$ coordinates. Let $\theta$ be the car's rotation coordinate. Let $v_{max}$ be the car's maximum speed and $\phi_{max}$ be the car's maximum steering angle. Let $a_0$ be the first dimension of the action, controlling the car's speed. Let $a_1$ be the second dimension of the action, controlling the car's steering. 

Then the next state described by the variables $x, y, \theta$ is found by numerically integrating the differential equation 
\begin{align*}
    \frac{dx}{dt}(t) = a_0 \cos{\theta(t)}
    \frac{dy}{dt}(t) = a_0 \sin{\theta(t)}
    \frac{d \theta}{dt}(t) = a_1 
\end{align*}
from time $t$ to time $t+1$. The $x(t+1), y(t+1), \theta(t+1)$ found at the end of this numerican integration is the next state. 


\subsubsection{Quadcopter}
The Quadcopter environment is taken from \citet{Sivaramakrishnan_Carver_Tangirala_Bekris_2023}. In this environment, the agent must navigate a quadcopter around a series of pillars to reach a goal. Episodes are 30 steps long. The reward function is 1 in the goal region, and 0 at all other states.  If an agent reaches the goal before the end of the episode, the episode ends and the agent receives a reward of 1 for each remaining time step left in the episode. The dynamics of this environment are given by a Mujoco simulation. 

\subsection{Additional Experiments} \label{Appendix: Additional Experiments}

In the experiments section, we reported that Volume-MCTS outperformed SST on reward, but not success rate. Here we provide addition details on this finding. 

\begin{figure}[ht]
\centering
\begin{subfigure}[t]{.48\linewidth}
\centering
\includegraphics[width=\linewidth, valign=t]{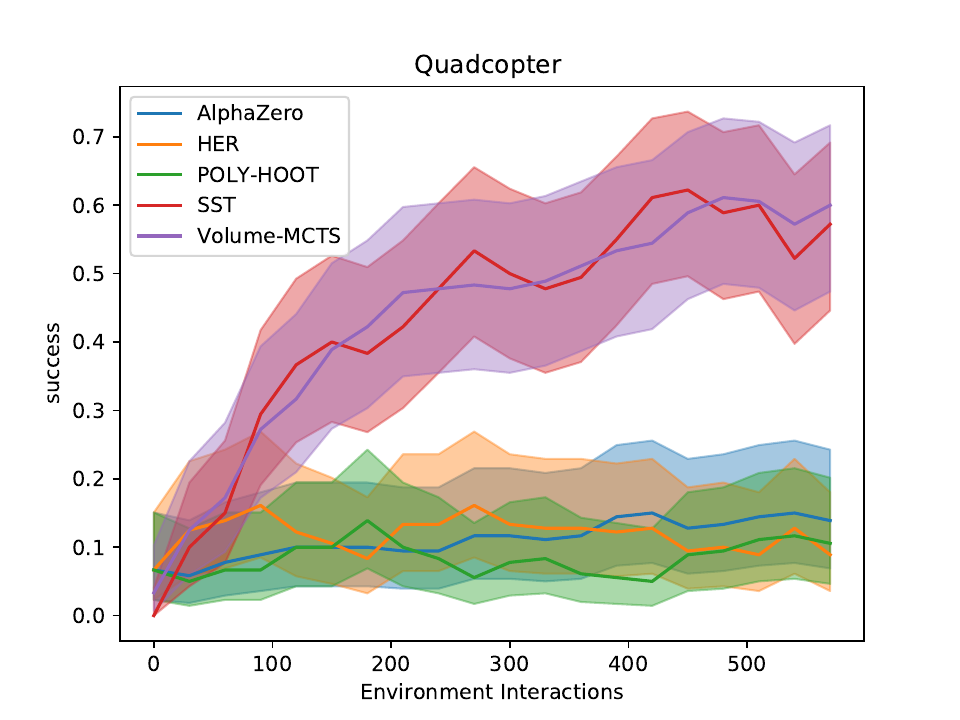}
\caption{\footnotesize Success rate as a function of total environmental interactions}
\label{dubins_pre_train}
\end{subfigure}
\begin{subfigure}[t]{.48\linewidth}
\centering
\includegraphics[width=\linewidth, valign=t]{graphics/Quadcopter_R_.pdf}
\caption{\footnotesize Reward as a function of total environmental interactions}
\label{dubins_post_train}
\end{subfigure}
\caption{\footnotesize Reward and success rate on Quadcopter environment}
\end{figure} 

Here, success is defined as reaching the goal within the 30-step episode. The goal state is treated as a s
SST and Volume-MCTS reach the goal roughly the same fraction of the time. However,

\end{document}